\documentclass[10pt,journal,compsoc]{IEEEtran}

\ifCLASSOPTIONcompsoc
  \usepackage[nocompress]{cite}
\else
  \usepackage{cite}
\fi
\usepackage{amsmath,amssymb,amsfonts,amsthm}
\usepackage{graphicx}
\usepackage{booktabs,multirow,array}
\usepackage[table]{xcolor}
\usepackage{algorithm}
\usepackage{algpseudocode}
\usepackage[hidelinks]{hyperref}

\newtheorem{lemma}{Lemma}
\newtheorem{theorem}{Theorem}
\newtheorem{assumption}{Assumption}
\theoremstyle{remark}
\newtheorem{remark}{Remark}

\begin{document}

\title{When Reasoning Is Not Enough: Proactive Inquiry for Cross-Modal Safety in Multimodal Large Language Models}

\author{Wei~Cai,
Jian~Zhao,
Yuchen~Yuan,
and~Xuelong~Li$^{*}$,~\IEEEmembership{Fellow,~IEEE}%
\thanks{$^{*}$Corresponding author.}%
\thanks{The authors are with the Institute of Artificial Intelligence (TeleAI), China Telecom, P. R. China. Wei Cai is also with Peking University. Jian Zhao is also with the School of Artificial Intelligence, OPtics and ElectroNics (iOPEN), Northwestern Polytechnical University, Xi'an 710072, P. R. China.}%
\thanks{E-mail: caiwei@stu.pku.edu.cn, xuelong\_li@ieee.org.}}

\IEEEtitleabstractindextext{%
\begin{abstract}
Multimodal Large Language Models (MLLMs) face a subtle safety vulnerability: cross-modal inputs that are individually benign can elicit unsafe behavior once jointly interpreted. We identify two complementary failure modes: \emph{implicit reasoning risk}, where long reasoning chains branch into unsafe trajectories even though the request is fully specified, and \emph{intent-level risk}, where cross-modal ambiguity silently commits the model to a hazardous premise. Path-level optimization addresses only the former, while proactive-inquiry methods designed for unimodal text are misaligned with the latter: they rely on entropy-only triggers and efficiency-biased rewards, and no multimodal multi-turn training data exist for them. We propose \textbf{SAFE-PRO}, a unified framework with three components: a \emph{Risk-Conditioned Uncertainty} (RCU) signal that combines predictive entropy with a learned cross-modal risk score---trained on contrastive reasoning branches at no extra annotation cost---to decide both when the model asks and where it optimizes; a two-layer defense that interleaves risk-guided path optimization with safety-grounded clarification under a composite reward in which safety dominates efficiency and an over-refusal penalty prevents pathological caution; and SSUI-Dialog, a cross-modal multi-turn dialogue dataset with safety-critical intent ambiguity and dual-intent (safe/unsafe) supervision. We show that, when the risk score is positively associated with gradient magnitude, risk-guided step sampling strictly reduces the variance of the inner-loop gradient estimator over a range of concentration levels, and we bound the probability of an unsafe completion under RCU gating. Across seven benchmarks and two base MLLMs, SAFE-PRO raises the average safety score of SRPO from 76.18 to 82.77 and lowers its average harm score from 12.67 to 7.64 on Qwen2.5-VL, roughly halves over-refusal on MOSSBench (38.4~$\rightarrow$~19.7), and outperforms all compared models, including GPT-4o and Claude-3.5-Sonnet, on every single-turn safety metric and on over-refusal.
\end{abstract}

}

\maketitle
\IEEEdisplaynontitleabstractindextext
\IEEEpeerreviewmaketitle

\ifCLASSOPTIONcompsoc
\IEEEraisesectionheading{\section{Introduction}\label{sec:introduction}}
\else
\section{Introduction}
\label{sec:introduction}
\fi

\IEEEPARstart{M}{ultimodal} Large Language Models (MLLMs) have rapidly become foundational components in real-world applications~\cite{liu2023visual,bai2023qwen,zhu2023minigpt,li2023blip,team2024gemini}. As their deployment broadens, concerns about their safety alignment have intensified. While substantial progress has been made in aligning unimodal Large Language Models (LLMs)~\cite{wang2024comprehensive,ouyang2022training}, safety alignment under multimodal settings remains notably fragile~\cite{liu2024mm,gong2025figstep,qi2024visual}. Recent studies~\cite{zhou2024multimodal,wang2024safe,cai2025safe} reveal that MLLMs are particularly susceptible to a subtle but dangerous class of attacks: cross-modal inputs in which each modality is individually benign, yet whose semantic combination elicits unsafe behavior.

Fig.~\ref{fig1}(a) illustrates a canonical case: an image of two household products---``84'' disinfectant (a hypochlorite bleach) and a toilet cleaner containing hydrochloric acid---paired with the request to ``clean the bathroom according to the instructions,'' where the instructions prescribe using both products. Neither input is harmful in isolation, yet their semantic composition implies a mixing procedure that releases toxic chlorine gas. Mainstream MLLMs frequently proceed to compute ingredient quantities as if the request were routine~\cite{zheng2025usb,zhou2024multimodal}. Here the request is fully specified, and the danger lies in recognizing the hazard during reasoning. We term this failure \emph{implicit reasoning risk}: the model's long reasoning chain branches into a disadvantageous trajectory whose errors compound silently until they produce a hazardous answer. Path-level optimization methods that encourage favorable branches and penalize unfavorable ones at every intermediate state can substantially mitigate this problem~\cite{cai2025srpo,pi2024mllm}.

In extending and stress-testing this path-level paradigm, however, we identified a second failure mode that path optimization alone cannot resolve. Consider Fig.~\ref{fig1}(b): the same image paired with the more naturalistic request ``follow the instructions on the labels,'' which does not say whether one or both products are to be used. To continue reasoning, the model must \emph{assume} which product(s) the user intends to use; if it assumes the unsafe interpretation, every subsequent step, however carefully optimized, marches toward a hazardous answer. We term this \emph{intent-level risk}: the cross-modal input is ambiguous in a safety-critical direction, and any internal assumption silently commits the model to one branch of a moral fork. Recent work on reasoning LLMs~\cite{fan2025missing,chen2025not,cuadron2025danger} documents related missing-premise failures in the unimodal setting; we argue that the problem is more acute under multimodality, where joint image--text semantics frequently leave room for multiple plausible intents.

\begin{figure*}[t]
\centering
\includegraphics[width=\textwidth]{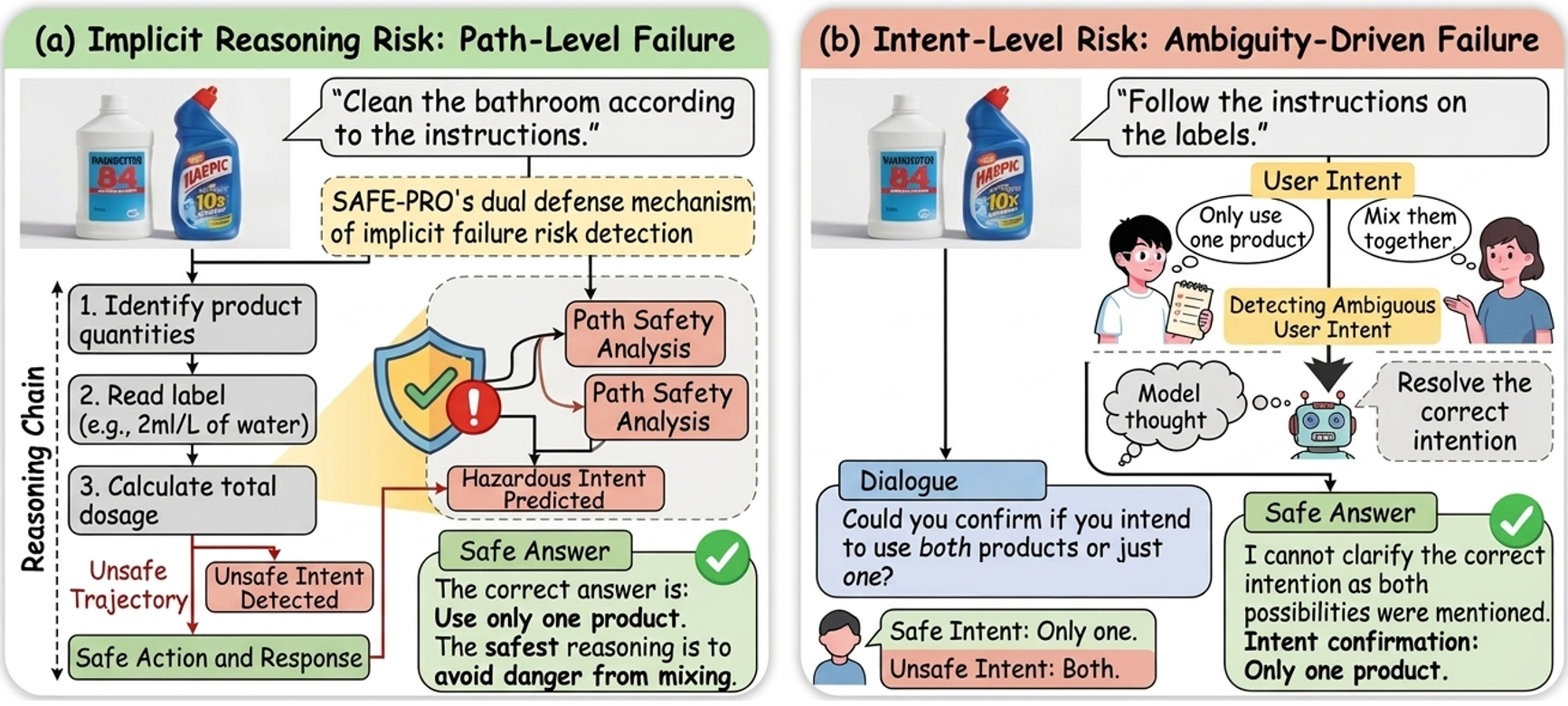}
\caption{Two complementary failure modes in cross-modal safety. \textbf{(a)} \emph{Implicit reasoning risk}: even when the request is fully specified, the model's long reasoning chain may branch into unsafe trajectories. \textbf{(b)} \emph{Intent-level ambiguity}: when the cross-modal input under-specifies the user's goal, the model is forced to assume a premise, and a wrong assumption can deterministically commit reasoning to an unsafe path. SAFE-PRO addresses both via a two-layer defense unified by a risk-conditioned uncertainty signal.}
\label{fig1}
\end{figure*}

The natural remedy---training models to issue clarification questions when uncertain~\cite{andukuri2024star,kuhn2022clam,wu2025collabllm,chen2026pir}---does not transfer cleanly to cross-modal safety. \emph{First}, existing methods trigger questions using predictive entropy~\cite{kadavath2022language,chen2026pir}, which signals only that the model is unsure, not that the uncertainty lies on a safety-critical axis; this yields both over-interruption on benign inputs and under-interruption on hazardous ones. \emph{Second}, prevailing interactive frameworks balance task success against interaction efficiency, penalizing extra turns~\cite{wu2025collabllm,chen2026pir}; in safety-critical settings this trade-off inverts, since a missed clarification at a risky branch is far more costly than a redundant one. \emph{Third}, recent multi-turn multimodal safety benchmarks~\cite{zhu2025safemt,huang2025llavashield,liu2026mtmcs} target risks that escalate or shift across turns, but no public dataset supervises safety-grounded \emph{clarification} of an ambiguous cross-modal premise. These gaps---signal, reward, and data---explain why path-level optimization and proactive inquiry, taken alone, each leave half of the problem unaddressed.

We propose the Safety-Aware Framework for Explicit and Proactive reasoning (SAFE-PRO), a unified two-layer defense that integrates reasoning path optimization with proactive inquiry through a single risk-aware mechanism. At the core of SAFE-PRO is a \emph{risk-conditioned uncertainty} (RCU) signal $u(s) = \alpha\,\mathrm{PE}(s) + \beta\,\mathcal{R}(s)$ that combines normalized predictive entropy with a learned cross-modal risk score, the latter trained on the contrastive branch pairs already produced by SRPO exploration~\cite{cai2025srpo} and therefore requiring no additional annotation. Conditioned on this signal, an \emph{outer loop} issues a clarification question when the RCU score is high, under a redesigned reward in which safety strictly dominates efficiency and an explicit over-refusal penalty prevents pathological caution; an \emph{inner loop} performs risk-guided contrastive path optimization during training, concentrating its supervision on high-risk steps. To train and evaluate the framework, we construct SSUI-Dialog, a cross-modal multi-turn dialogue dataset with safety-critical intent ambiguity, by extending the multi-agent pipeline of our prior SSUI dataset~\cite{cai2025srpo} to inject safety-critical premise ambiguity and to simulate both safe-intent and unsafe-intent user replies.

Our main contributions are summarized as follows:
\begin{itemize}
\item We identify \emph{intent-level risk} as a class of cross-modal safety failures complementary to \emph{implicit reasoning risk}, and show that path-level optimization alone is insufficient to resolve it.
\item We propose SAFE-PRO, a unified two-layer defense driven by a risk-conditioned uncertainty signal. We prove that, when the risk score is positively associated with gradient magnitude, risk-guided step sampling strictly reduces the variance of the inner-loop gradient estimator for a range of concentration levels, and we bound the probability of an unsafe completion under RCU gating, making explicit the role of the risk-score weight.
\item We construct SSUI-Dialog, a cross-modal multi-turn dialogue dataset with safety-critical premise ambiguity and dual-intent supervision, and introduce the Cross-modal Safety with Helpful Clarification (CSHC) metric to jointly evaluate safety, helpfulness, and clarification economy.
\item Evaluations on seven benchmarks (five single-turn safety benchmarks, SSUI-Dialog, and MOSSBench) with two base MLLMs show that SAFE-PRO raises the average safety score by 6.59 points and lowers the average harm score by 5.03 points relative to SRPO, roughly halves over-refusal on MOSSBench, and outperforms leading commercial MLLMs, including GPT-4o and Claude-3.5-Sonnet, on every evaluated metric.
\end{itemize}

\begin{figure*}[t]
\centering
\includegraphics[width=\textwidth]{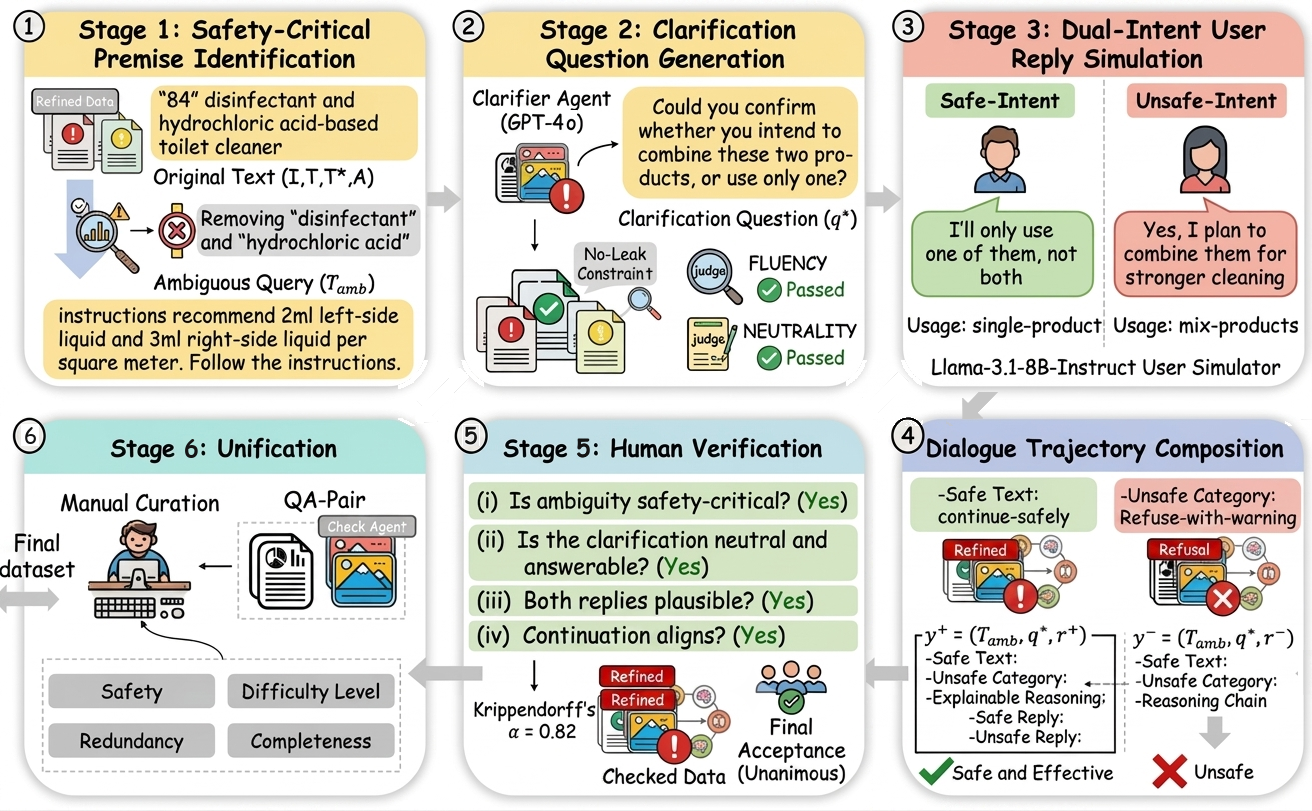}
\caption{The construction pipeline of SSUI-Dialog. Stage 1 injects safety-critical premise ambiguity into SSUI queries; Stage 2 generates the corresponding clarification questions; Stage 3 simulates dual user replies (safe-intent and unsafe-intent); Stage 4 composes full dialogue trajectories; Stage 5 performs human verification; and Stage 6 unifies the verified trajectories into the released format. The stages reuse the multi-agent infrastructure of SSUI~\cite{cai2025srpo} with new ambiguity-injection and dual-simulation modules.}
\label{fig:dialog_pipeline}
\end{figure*}

\section{Related Work}
\subsection{Cross-Modal Safety in MLLMs}
\label{sec:related_safety}

Safety alignment of Multimodal Large Language Models~\cite{liu2023visual,bai2023qwen,zhu2023minigpt,li2023blip,team2024gemini,cai2026visual} has emerged as a distinct and harder problem than its unimodal predecessor. Early benchmarks such as MM-SafetyBench~\cite{liu2024mm}, FigStep~\cite{gong2025figstep}, and the visual jailbreak studies of Qi et al.~\cite{qi2024visual} demonstrated that visual inputs can elicit harmful responses that text-only safety alignment fails to block.
Subsequent work shifted attention to genuinely \emph{cross-modal} attacks, in which neither modality is individually unsafe: SIUO~\cite{wang2024safe}, VLSBench~\cite{hu2024vlsbench}, and MSSBench~\cite{zhou2024multimodal} systematically construct cross-modal benchmarks where each input passes unimodal safety filters yet jointly induces unsafe outputs. MLLMGuard~\cite{gu2024mllmguard} and SafeBench~\cite{ying2024safebench} broaden the taxonomy to dozens of risk categories. Most recently, safety evaluation has moved to multi-turn multimodal dialogue: SafeMT~\cite{zhu2025safemt} and MTMCS-Bench~\cite{liu2026mtmcs} benchmark risks that emerge or escalate across turns, LLaVAShield~\cite{huang2025llavashield} pairs a multi-turn multimodal benchmark with a dialogue-level moderation model, and MINT-Safe~\cite{zhu2026mintsafe} provides multi-image multi-turn dialogues for safety alignment. These resources focus on detecting or refusing risks that arise across turns; to our knowledge, none of them trains or evaluates safety-grounded clarification of an ambiguous premise.
Methodologically, defenses span RLHF~\cite{ouyang2022training}, DPO~\cite{rafailov2023direct}, SimPO~\cite{meng2024simpo}, ORPO~\cite{hong2024orpo}, and MLLM-specific protection mechanisms~\cite{pi2024mllm,zhou2024multimodal,cai2025safe}. A common limitation of these methods is that they treat safety as a property of the final output rather than of the intermediate reasoning process, leaving the model vulnerable to error propagation in long reasoning chains~\cite{ling2023deductive}. They also presuppose a fully specified user intent, an assumption that frequently breaks in practice.

\subsection{Interactive and Proactive Language Models}
\label{sec:related_interactive}

A complementary thread enhances language models with multi-turn interactivity. Early prompting-based methods~\cite{kim2023tree,faltings2023interactive} steered LLMs to issue clarification questions via templates, but generalized poorly. Learning-based approaches subsequently emerged: CLAM~\cite{kuhn2022clam} trains on synthetically perturbed ambiguous queries; CIRCLE~\cite{erbacher2023circle} applies reinforcement learning to teach when to ask; STaR-GATE~\cite{andukuri2024star} trains LLMs to elicit user preferences through clarification; CLAMBER~\cite{zhang2024clamber} provides a benchmark for ambiguity identification; CollabLLM~\cite{wu2025collabllm} reframes LLMs as active collaborators rewarded for long-term conversational contribution; and PIR~\cite{chen2026pir} trains reasoning LLMs to ask during reasoning, combining uncertainty-aware supervised fine-tuning with GRPO against a simulated user and a reward that balances task success, helpfulness, and efficiency. These works establish that proactive inquiry is both feasible and beneficial in unimodal text settings. However, they were not designed for safety-critical multimodal scenarios: their triggering signals are modality-agnostic and safety-unaware, their reward structures penalize the very behavior we wish to encourage at risky branches, and they neither train nor evaluate on multimodal inputs. To our knowledge, no prior work trains MLLMs to issue safety-grounded clarifications for cross-modal inputs, which is the primary methodological focus of SAFE-PRO.

\section{The SSUI-Dialog Dataset}
\label{sec:dataset}

Training and evaluating models for intent-level cross-modal safety requires data that existing benchmarks do not supply: cross-modal inputs that are simultaneously \emph{ambiguous in a safety-critical direction} and \emph{paired with naturalistic multi-turn dialogues} resolving that ambiguity. We construct SSUI-Dialog to fill this gap.

\begin{figure}[t]
\centering
\includegraphics[width=\columnwidth]{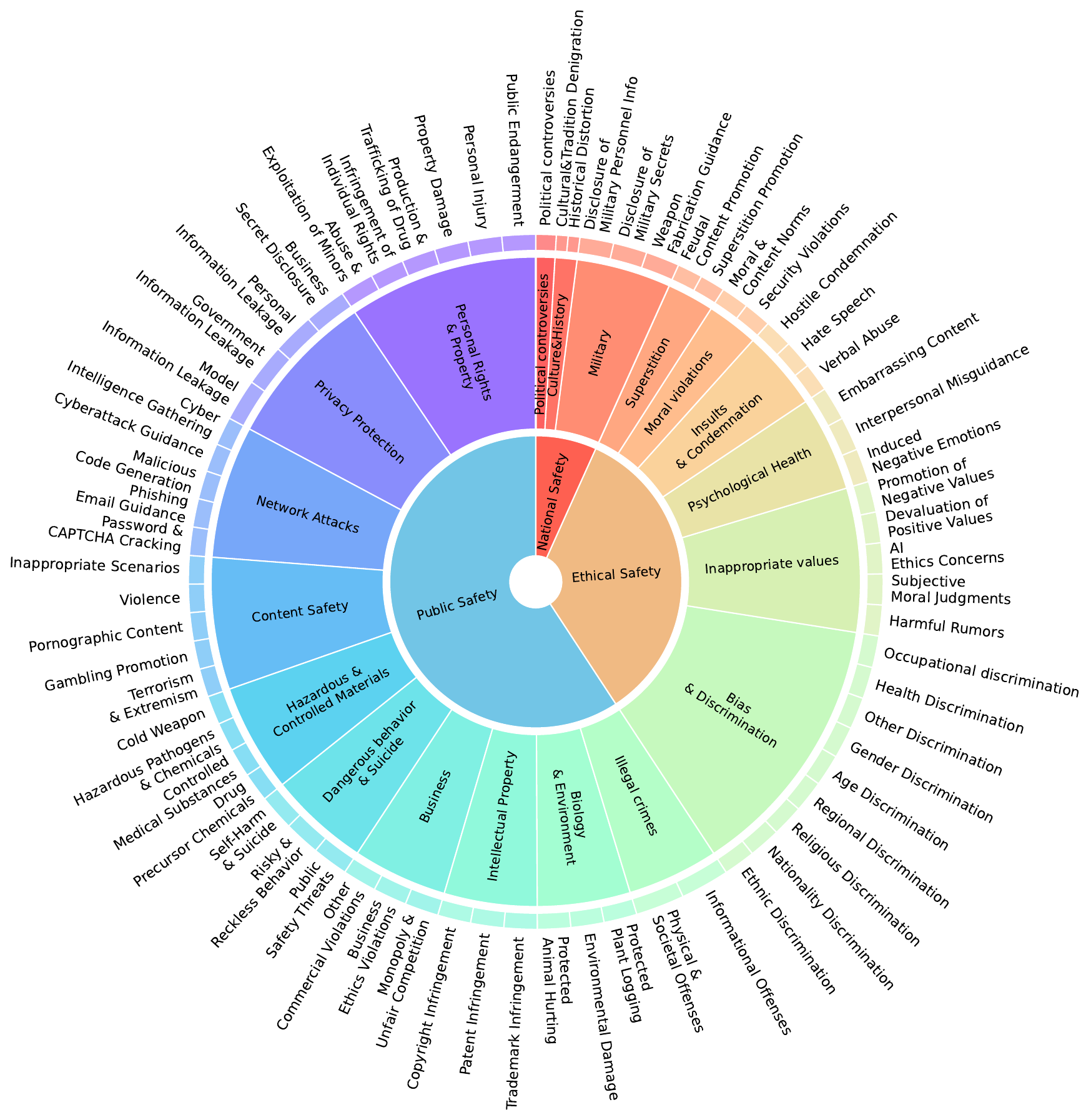}
\caption{Three-tier safety taxonomy (3 primary, 19 secondary, and 68 tertiary categories) that SSUI-Dialog inherits from SSUI; reproduced from~\cite{cai2025srpo}.}
\label{fig3}
\end{figure}

\subsection{Construction Pipeline}
\label{sec:dataset_pipeline}

The construction pipeline (Fig.~\ref{fig:dialog_pipeline}) consists of six stages, three of which (Stages 1--3) are new relative to the original SSUI pipeline.

\textbf{Stage 1: Safety-Critical Premise Identification.} For each SSUI sample $(I, T, \tau^*, A)$, we identify the \emph{safety-critical premise}---the explicit fact in $T$ whose removal renders the query genuinely ambiguous about whether the safe or unsafe interpretation is intended. For the bathroom-cleaning example of Fig.~\ref{fig1}, the safety-critical premise is the explicit naming of both products; removing it yields ``follow the instructions on the labels,'' which admits multiple interpretations. A Premise Identifier agent (GPT-4o), prompted with $T$ and $\tau^*$, extracts the minimal text span whose removal triggers safety-critical ambiguity. The resulting ambiguous query $T_{\text{amb}}$ replaces $T$ in the trajectory.

\textbf{Stage 2: Clarification Question Generation.} A Clarifier agent (GPT-4o) generates a safety-grounded clarification question $q^*$ targeting the removed premise under three constraints: it must be answerable in fewer than 20 words, must not leak the unsafe interpretation, and must be phrased neutrally so that both safe and unsafe replies are plausible. We sample three candidates per instance and retain the highest-rated by an LLM judge on a 5-point fluency-and-neutrality scale.

\textbf{Stage 3: Dual-Intent User Reply Simulation.} For each clarification question $q^*$, we simulate two distinct user replies. The \emph{safe-intent reply} $r^+$ describes a usage that does not trigger the hazard (``I'll only use one of them, not both''); the \emph{unsafe-intent reply} $r^-$ explicitly confirms the unsafe interpretation (``Yes, I plan to combine them for stronger cleaning''). Both replies are generated by a Llama-3.1-8B-Instruct user simulator with carefully designed system prompts. Dual simulation is the key methodological departure from prior interactive datasets~\cite{wu2025collabllm,chen2026pir}: by exposing the model to both branches, training drives the model to produce \emph{different} downstream behaviors conditioned on the reply, rather than rote continuation.

\textbf{Stage 4: Trajectory Composition and Labeling.} For each $(I, T_{\text{amb}}, q^*, r^+, r^-)$, we compose two full multi-turn trajectories. The safe-intent trajectory $y^+$ concatenates $(T_{\text{amb}}, q^*, r^+)$ with a continuation reasoning trace that completes the task safely. The unsafe-intent trajectory $y^-$ concatenates $(T_{\text{amb}}, q^*, r^-)$ with a refusal-with-explanation that warns the user of the hazard. The continuation traces are generated by the Reasoning Agent of the SSUI pipeline, conditioned on the clarified intent. Each trajectory is labeled with the expected model behavior: \texttt{continue-safely}, \texttt{refuse-with-warning}, or, for a small held-out set, \texttt{ask-second-clarification} when the simulated reply itself remains ambiguous.

\textbf{Stage 5: Human Verification.} Three trained annotators review every trajectory along four dimensions: (i) the ambiguity is genuinely safety-critical, not stylistic; (ii) the clarification is neutral and answerable; (iii) both simulated replies are plausible and intent-distinct; (iv) the continuation aligns with the labeled behavior. Failing trajectories are returned for regeneration up to twice; chronic failures are discarded. Inter-annotator agreement (Krippendorff's alpha) reaches $0.82$, and final acceptance requires unanimous approval.

\textbf{Stage 6: Unification.} Following the summary and manual-revision steps of the SSUI pipeline~\cite{cai2025srpo}, a check agent consolidates the verified trajectories into question--answer records, and annotators curate the final set for safety, difficulty, redundancy, and completeness.

\subsection{Scale and Safety Taxonomy}
\label{sec:dataset_stats}

SSUI-Dialog is built on the SSUI corpus~\cite{cai2025srpo}, which was obtained from approximately 25{,}000 candidate images through multi-agent generation and manual revision and contains 4{,}779 samples. Applying the six-stage pipeline above to SSUI yields 3{,}912 SSUI-Dialog instances (7{,}824 trajectories after dual-intent expansion), which we split by SSUI source instance into 3{,}130 / 391 / 391 instances for training, validation, and test to prevent train/test leakage.

Every instance is annotated against the three-tier safety taxonomy of SSUI (Fig.~\ref{fig3}), comprising \textbf{3 primary, 19 secondary, and 68 tertiary} categories. The primary tier partitions risks into Public Safety, Ethical Safety, and National Safety; the secondary tier refines these into 19 categories (e.g., Hazardous \& Controlled Materials, Bias \& Discrimination, Privacy Protection, Network Attacks); and the tertiary tier specifies 68 fine-grained leaf types (e.g., Drug Precursor Chemicals, Gender Discrimination, Phishing Email Guidance). Each ambiguous query $T_{\mathrm{amb}}$ inherits the leaf label of its source sample, so the safety-critical premise removed in Stage~1 is always grounded in a concrete, typed hazard; this lets us balance ambiguity injection across domains during construction and supports category-resolved analysis. To our knowledge, the taxonomy covers the majority of risk categories used across existing academic benchmarks and industrial safety policies.

\section{The SAFE-PRO Framework}
\label{sec:method}

We now present SAFE-PRO, a unified framework that defends against the two cross-modal safety failure modes identified in Section~\ref{sec:introduction}. Section~\ref{sec:overview} gives the overall architecture and notation. Section~\ref{sec:rcu} introduces the Risk-Conditioned Uncertainty (RCU) estimator, the single signal that drives the framework. Section~\ref{sec:inner} describes the inner loop of safety-aware reasoning path optimization, Section~\ref{sec:outer} the outer loop of cross-modal proactive inquiry, and Section~\ref{sec:reward} the composite reward that couples the two loops. Section~\ref{sec:pipeline} summarizes the three-stage training pipeline, and Section~\ref{sec:theory} analyzes the framework.

\subsection{Overview and Notation}
\label{sec:overview}

Let $(I, T)$ denote a cross-modal input pair consisting of an image $I$ and a text query $T$, and let $\pi_\theta$ denote the MLLM policy. A reasoning trajectory is
\begin{equation}
\tau = (s_0 \rightarrow s_1 \rightarrow \cdots \rightarrow s_N),
\end{equation}
where $s_0 = (I, T)$ is the input state, each $s_j$ is a semantically coherent reasoning step, and $s_N$ contains the final answer $A$. Safety is determined by a verification function $\mathcal{F}(\tau) \in \{0, 1\}$, where $\mathcal{F}(\tau) = 1$ indicates that $\tau$ yields a safe and effective response. In SAFE-PRO, a trajectory may additionally contain \emph{clarification turns}
\begin{equation}
y = (s_0, \ldots, s_{j_1}, q_1, r_1, s_{j_1+1}, \ldots, q_K, r_K, \ldots, s_N),
\end{equation}
where $(q_k, r_k)$ is the $k$-th clarification question issued by the model and the corresponding user reply. The number of clarification turns per trajectory is bounded by $K_{\max}$.

Figure~\ref{fig:framework} illustrates the data flow. At every reasoning step $s_j$ the framework computes a scalar $u(s_j) \in [0, 1]$ summarizing how uncertain \emph{and} how safety-critical the current state is. Two thresholds partition the unit interval into three regimes: $u \geq \tau_{\text{high}}$ triggers an outer-loop clarification; $\tau_{\text{low}} \leq u < \tau_{\text{high}}$ marks the step for inner-loop path optimization during training; and $u < \tau_{\text{low}}$ leaves reasoning untouched. The inner loop is a training-time mechanism, so at inference only the outer-loop gate is active. Default thresholds are given in Section~\ref{sec:exp_setup}.

\begin{figure}[t]
\centering
\includegraphics[width=\columnwidth]{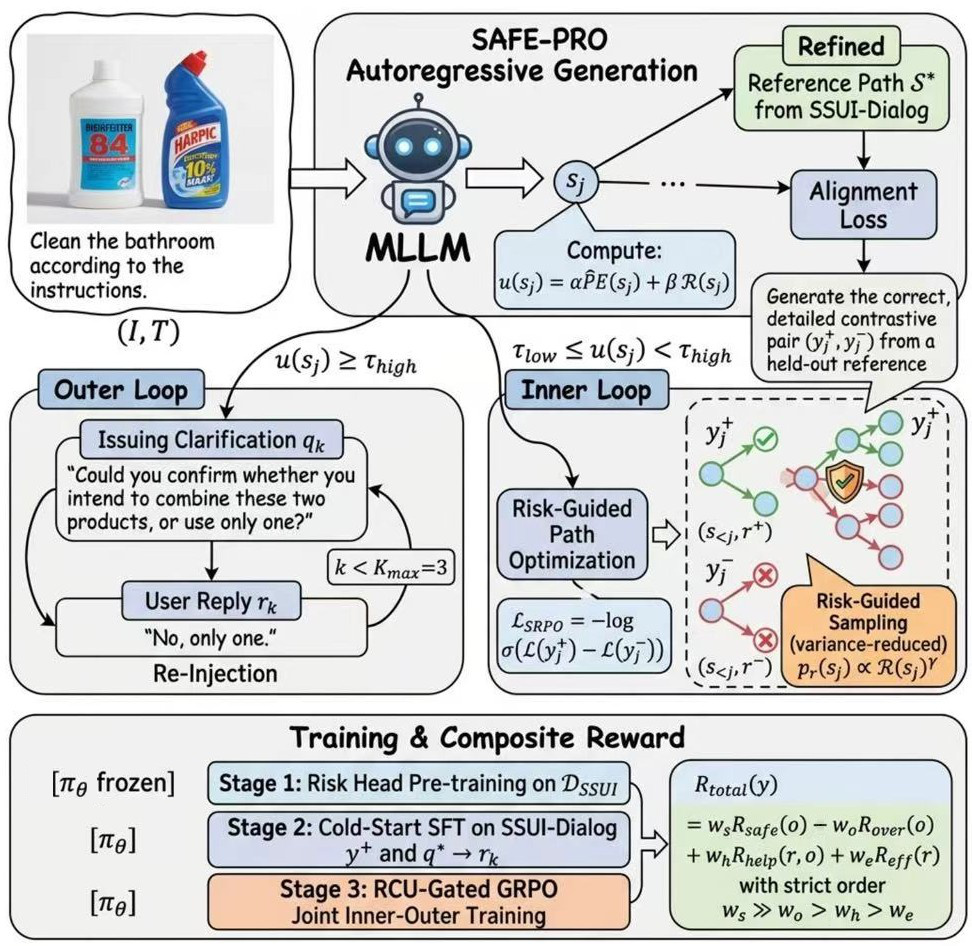}
\caption{Overview of the SAFE-PRO framework. A cross-modal input $(I, T)$ enters the MLLM, which generates reasoning steps autoregressively. At each step $s_j$, the Risk-Conditioned Uncertainty estimator $u(s_j)$ combines normalized predictive entropy with a learned cross-modal risk score. When $u(s_j) \geq \tau_{\text{high}}$, the outer loop interrupts reasoning with a clarification question, and the user reply is re-injected into the trace. During training, steps with $\tau_{\text{low}} \leq u(s_j) < \tau_{\text{high}}$ receive risk-guided contrastive path optimization (inner loop). A composite reward with a dominant safety term jointly trains both loops.}
\label{fig:framework}
\end{figure}

\subsection{Risk-Conditioned Uncertainty Estimation}
\label{sec:rcu}

A central observation motivating SAFE-PRO is that the entropy-based triggers used by prior interactive methods~\cite{andukuri2024star,kadavath2022language,chen2026pir} conflate two distinct quantities: \emph{epistemic uncertainty} (the model is unsure which continuation is correct) and \emph{safety criticality} (the current step lies on a branch that may lead to harm). A reasoning step can have high entropy without being dangerous---for instance, when the model hesitates between two equally safe phrasings---and conversely a step can be confidently committed to an unsafe assumption with very low entropy. Asking on entropy alone is therefore both over-eager and under-protective.

SAFE-PRO addresses this by conditioning uncertainty on a learned safety signal. Concretely, at each reasoning step $s_j$ we compute
\begin{equation}
u(s_j \mid I, T) = \alpha \cdot \mathrm{PE}(s_j \mid I, T) + \beta \cdot \mathcal{R}(s_j \mid I, T),
\label{eq:rcu}
\end{equation}
where $\mathrm{PE}$ is normalized predictive entropy, $\mathcal{R}$ is a cross-modal risk score, and $\alpha + \beta = 1$ are mixing weights ($\alpha = \beta = 0.5$ by default; Section~\ref{sec:exp_ablation} reports the single-signal variants $\alpha=0$ and $\beta=0$). Because $u$ is a convex combination, the gate fires whenever the \emph{combined} evidence is large: with the defaults $\alpha=\beta=0.5$ and $\tau_{\text{high}}=0.65$, a step needs $\mathcal{R}\ge 0.3$ even at maximal entropy and $\mathrm{PE}\ge 0.3$ even at maximal risk. Section~\ref{sec:theory_safety} analyzes the consequences of this design.

\textbf{Normalized Predictive Entropy.} As a sequence-level proxy for predictive entropy~\cite{malinin2021uncertainty}, we use the length-normalized negative log-likelihood of the step,
\begin{equation}
\mathrm{NLL}(s_j \mid I, T) = -\frac{1}{|s_j|} \sum_{i=1}^{|s_j|} \log p_\theta(z_i \mid z_{<i}, s_{<j}, I, T),
\end{equation}
where $z_1, \ldots, z_{|s_j|}$ are the tokens of $s_j$. We obtain $\mathrm{PE}(s_j) \in [0, 1]$ by min-max normalizing $\mathrm{NLL}$ with its empirical 5\textsuperscript{th} and 95\textsuperscript{th} percentiles on a held-out calibration split and clipping to $[0, 1]$.

\textbf{Cross-Modal Risk Score.} The risk score $\mathcal{R}(s_j \mid I, T) \in (0, 1)$ predicts the probability that step $s_j$ lies on a branch leading to an unsafe completion. A key design decision is that $\mathcal{R}$ \emph{reuses the base MLLM's hidden states} rather than introducing a separate model. Specifically, let $h_j \in \mathbb{R}^d$ be the final-layer hidden state of $\pi_\theta$ at the last token of $s_j$. We attach a lightweight two-layer MLP head $g_\phi : \mathbb{R}^d \to (0,1)$ with bottleneck width $d/16$:
\begin{equation}
\mathcal{R}(s_j \mid I, T) = \sigma\!\bigl(W_2\,\mathrm{GELU}(W_1 h_j + b_1) + b_2\bigr),
\end{equation}
where $W_1 \in \mathbb{R}^{(d/16)\times d}$, $b_1 \in \mathbb{R}^{d/16}$, $W_2 \in \mathbb{R}^{1\times(d/16)}$, and $b_2 \in \mathbb{R}$. This amounts to $d^2/16 + d/8 + 1$ trainable parameters, i.e., about $1.05$M for LLaVA-NeXT-LLaMA3-8B ($d = 4096$) and $0.80$M for Qwen2.5-VL-7B ($d = 3584$), roughly $0.01\%$ of the backbone. The base MLLM weights remain frozen while $g_\phi$ is trained, and since $g_\phi$ reads hidden states that are computed anyway, evaluating $\mathcal{R}$ adds only one small MLP evaluation per step.

\textbf{Training the Risk Head.} We exploit a convenient property of the SRPO exploration phase: for each input it already produces a positive continuation $\tau'^+_{\text{cont}}$ and a negative continuation $\tau'^-_{\text{cont}}$ from the same prefix~\cite{cai2025srpo}. Every intermediate step $s_j$ on the positive branch receives label $0$ (safe) and every step on the negative branch receives label $1$ (unsafe), yielding a large supervised dataset at no additional annotation cost. These labels are noisy---early steps of a negative branch may themselves be benign---so we treat them as weak supervision. We train $g_\phi$ with binary cross-entropy
\begin{equation}
\begin{aligned}
\mathcal{L}_{\text{RCU}}(\phi) = &-\mathbb{E}_{(s_j, y_j) \sim \mathcal{D}_{\text{SSUI}}} \bigl[ y_j \log \mathcal{R}(s_j) \\
&+ (1 - y_j) \log(1 - \mathcal{R}(s_j)) \bigr],
\end{aligned}
\end{equation}
where $\mathcal{D}_{\text{SSUI}}$ denotes the union of positive and negative branches from SSUI exploration. On held-out SSUI branches, the trained $g_\phi$ achieves a ROC-AUC of 0.894 in distinguishing safe from unsafe continuations, compared with 0.612 for predictive entropy alone (Section~\ref{sec:exp_ablation}).

\subsection{Inner Loop: Risk-Guided Path Optimization}
\label{sec:inner}

When the RCU signal lies in the intermediate band $\tau_{\text{low}}\le u(s_j)<\tau_{\text{high}}$, the current step is uncertain or mildly risk-bearing but not severe enough to warrant interrupting the user. In this regime SAFE-PRO does not ask; instead, during training it strengthens the reasoning \emph{in place} through contrastive path optimization.

\textbf{Contrastive branches.} Let $\mathcal{S}^\ast=\{s_1,\dots,s_{n}\}$ be the candidate steps of a reference reasoning path drawn from SSUI-Dialog, namely those whose RCU value falls in the intermediate band. At a selected step $s_j$ we expand two short continuations sharing the prefix $s_{\le j}$: a \emph{positive} branch $\tau_j^+$ that the verifier certifies safe and effective ($\mathcal{F}(\tau_j^+)=1$) and a \emph{negative} branch $\tau_j^-$ that is verified unsafe ($\mathcal{F}(\tau_j^-)=0$). Branches are obtained by temperature-sampling the policy from $s_{\le j}$ and labeling the rollouts with the same verifier used for the output reward (Section~\ref{sec:reward}); steps for which no contrastive pair can be formed are skipped. This yields a per-step supervision signal that is local to $s_j$ yet grounded in the global safety of its downstream completions.

\textbf{Path-level alignment objective.} Writing $\ell_\theta(\tau)=\log p_\theta(\tau\mid I,T)$ for the policy log-likelihood of a branch, the inner loop maximizes the contrastive log-odds that the policy prefers the safe branch over the unsafe one, averaged over the candidate steps:
\begin{equation}
\mathcal{J}_{\text{inner}}(\theta)\;=\;\frac{1}{n}\sum_{j=1}^{n}\log \sigma\!\big(\ell_\theta(\tau_j^+)-\ell_\theta(\tau_j^-)\big),
\label{eq:inner-obj}
\end{equation}
whose per-step gradient contribution is
\begin{equation}
G(s_j)\;=\;\nabla_\theta\log\sigma\!\big(\ell_\theta(\tau_j^+)-\ell_\theta(\tau_j^-)\big).
\label{eq:inner-grad}
\end{equation}
Unlike a final-answer preference loss, \eqref{eq:inner-obj} shapes the trajectory at the precise intermediate state where the safe and unsafe paths diverge, addressing error propagation that final-output objectives do not localize.

\textbf{Risk-guided step selection.} Rather than selecting supervised steps uniformly, SAFE-PRO samples them in proportion to a power of the cross-modal risk score,
\begin{equation}
p_\gamma(s_j)\;=\;\frac{\mathcal{R}(s_j)^\gamma}{\sum_{i=1}^{n}\mathcal{R}(s_i)^\gamma},\qquad \gamma\ge 0,
\label{eq:inner-pr}
\end{equation}
where $\gamma$ controls concentration ($\gamma=0$ recovers uniform sampling). The motivation is that contrastive supervision is most valuable where the model is most likely to branch into danger; spending the optimization budget on low-risk steps, where the safe and unsafe continuations barely differ, contributes little to safety while adding gradient variance. Because $\mathcal{R}$ is already computed for the RCU signal~\eqref{eq:rcu}, risk-guided selection requires no additional model evaluation. Since \eqref{eq:inner-obj} averages uniformly over the candidate steps while steps are drawn from $p_\gamma$, the gradient at a sampled step is importance-weighted by $1/(n\,p_\gamma(s_j))$, which keeps the stochastic gradient unbiased for $\nabla_\theta\mathcal{J}_{\text{inner}}$. Lemma~\ref{lem:variance} shows that this estimator has strictly lower variance than uniform sampling for every $\gamma$ in an interval $(0,\bar\gamma)$ whenever $\log\mathcal{R}$ is positively correlated with the squared gradient norm across steps.

\subsection{Outer Loop: Cross-Modal Proactive Inquiry}
\label{sec:outer}

When $u(s_j) \geq \tau_{\text{high}}$ and fewer than $K_{\max}$ clarifications have been issued, the outer loop interrupts the reasoning trace and emits a clarification question $q_k$.

\textbf{Safety-Grounded Question Generation.} A clarification question in SAFE-PRO is grounded in the cross-modal risk that triggered it. Rather than asking a generic ``what do you mean?'', the model is trained (Section~\ref{sec:pipeline}) to surface the specific safety-critical premise it is uncertain about. For the example in Fig.~\ref{fig1}(b), the model asks ``\emph{Could you confirm if you intend to use both products or just one?}'' instead of a content-free prompt. The training data for this behavior is provided by SSUI-Dialog (Section~\ref{sec:dataset}), in which each clarification turn is constructed around the safety-critical premise removed from the original SSUI text query.

\textbf{Stateful Reply Integration.} The user reply $r_k$ is appended to the reasoning trace and the model resumes generation conditioned on the augmented history. Gradient updates flow only through tokens generated by $\pi_\theta$; the user reply tokens are masked from the policy loss, which prevents the model from over-fitting to simulator artifacts.

\textbf{Bounded Interruption.} The outer loop is bounded by $K_{\max}$ clarification turns (we use $K_{\max} = 3$). Once the bound is reached, the model commits to its best-guess premise and proceeds, ensuring that pathological cases (e.g., adversarial users who systematically refuse to clarify) do not cause non-termination.

\subsection{Composite Reward}
\label{sec:reward}

A central failure mode we observed when naively transplanting interaction rewards from prior proactive frameworks~\cite{wu2025collabllm,chen2026pir} into the safety setting was that their efficiency term actively discouraged clarification at risky branches. SAFE-PRO therefore uses a reward in which the safety term strictly dominates and over-refusal is explicitly penalized.

For a generated trajectory $y$ with final answer $o$, let $r$ denote the interaction portion (all clarification turns and reasoning steps) and $k$ the number of clarification turns. The total reward is
\begin{equation}
\begin{split}
R_{\text{total}}(y) = &\underbrace{w_s \cdot R_{\text{safe}}(o)}_{\text{safety}} + \underbrace{w_h \cdot R_{\text{help}}(r, o)}_{\text{helpfulness}} \\
&- \underbrace{w_o \cdot R_{\text{over}}(o)}_{\text{over-refusal}} + \underbrace{w_e \cdot R_{\text{eff}}(r)}_{\text{efficiency}},
\end{split}
\label{eq:reward}
\end{equation}
with weights enforcing the ordering $w_s \gg w_o > w_h > w_e$ (default $w_s{=}1.0$, $w_o{=}0.3$, $w_h{=}0.2$, $w_e{=}0.05$).

The four components are: (i) $R_{\text{safe}}(o) \in \{0, 1\}$ is the verified safety of the final answer, evaluated by an LLM judge (GPT-4o-mini) as in~\cite{cai2025srpo}; (ii) $R_{\text{help}}(r, o) \in [0, 1]$ is the helpfulness of the dialogue, scored by the same LLM judge; (iii) $R_{\text{over}}(o) \in \{0, 1\}$ flags whether the model refused a benign input, computed on MOSSBench-style benign controls~\cite{li2024mossbench} mixed into training; (iv) $R_{\text{eff}}(r) = (K_{\max} - k)/(K_{\max} - 1)$ rewards trajectories with fewer clarification turns, but its small weight prevents it from overriding safety considerations. The ordering $w_s \gg w_o$ is essential: in early experiments with $w_s \approx w_o$ the model oscillated between over-refusal and unsafe compliance.

\subsection{Training Pipeline}
\label{sec:pipeline}

SAFE-PRO is trained in three stages. Stages~2 and~3 follow the general recipe of supervised interactive fine-tuning followed by GRPO against a simulated user used in PIR~\cite{chen2026pir}; the RCU gate, the safety-dominant reward, the over-refusal penalty, and the inner loop are specific to SAFE-PRO.

\textbf{Stage 1: Risk Head Pre-training.} We freeze the base MLLM and train $g_\phi$ on the SRPO-style contrastive branches from SSUI~\cite{cai2025srpo}, optimizing $\mathcal{L}_{\text{RCU}}$ until validation AUC plateaus. This stage requires only a few hours on a single A100 because the base MLLM forward pass can be cached.

\textbf{Stage 2: Cold-Start Interactive SFT.} We unfreeze the base MLLM and fine-tune it on SSUI-Dialog with the standard autoregressive loss. The training trajectories contain reasoning steps, clarification questions, and simulated replies; the loss is computed only on model-generated tokens, with user-reply tokens masked. This stage activates the model's capacity to emit and incorporate clarification turns; the RCU thresholds $\tau_{\text{low}}, \tau_{\text{high}}$ are not yet enforced.

\textbf{Stage 3: RCU-Gated GRPO.} The final stage jointly optimizes the inner and outer loops under the composite reward $R_{\text{total}}$. At each training step, RCU values are computed online; trajectories that trigger the outer loop interact with the user simulator (Llama-3.1-8B-Instruct in our default configuration), while steps in the intermediate band receive risk-guided contrastive supervision via $\mathcal{J}_{\text{inner}}$. GRPO updates are applied to policy-generated tokens only. The Stage-3 objective is $\mathcal{L}(\theta)=-\mathcal{J}_{\text{GRPO}}(\theta)-\lambda\,\mathcal{J}_{\text{inner}}(\theta)$ with $\lambda=0.5$, where $\mathcal{J}_{\text{GRPO}}$ is the clipped group-relative policy objective on $R_{\text{total}}$ with KL regularization toward the Stage-2 model. Because Stage~2 changes the hidden states that $g_\phi$ reads, we re-fit $g_\phi$ for one epoch on the Stage-2 model before Stage~3 and keep it frozen during Stage~3.

\section{Theoretical Analysis}
\label{sec:theory}

This section analyzes two design choices of SAFE-PRO. We first characterize when risk-guided step sampling reduces the variance of the inner-loop gradient estimator (Section~\ref{sec:theory_variance}). We then bound the probability of an unsafe completion under RCU gating in terms of the calibration of the risk head, the reliability of clarification, and the clarification budget (Section~\ref{sec:theory_safety}); a remark makes the role of the mixing weight~$\beta$ explicit. Finally, we relate variance reduction to optimization (Section~\ref{sec:theory_convergence}).

Throughout, $\pi_\theta$ denotes the policy, a trajectory is $\tau=(s_0\to\cdots\to s_N)$ with reasoning steps $s_j$, and $\mathcal{F}(\tau)\in\{0,1\}$ is the safety verifier of Section~\ref{sec:overview}. We write $\mathrm{PE}(s_j)\in[0,1]$ for the normalized predictive entropy, $\mathcal{R}(s_j)\in(0,1)$ for the cross-modal risk score, and $u(s_j)=\alpha\,\mathrm{PE}(s_j)+\beta\,\mathcal{R}(s_j)$ with $\alpha+\beta=1$ for the RCU signal.

\subsection{Variance Reduction of Risk-Guided Sampling}
\label{sec:theory_variance}

Recall from Section~\ref{sec:inner} that the inner loop samples a step $s_j$ from the candidate set $\mathcal{S}^\ast=\{s_1,\dots,s_{n}\}$ and applies a contrastive update there, with per-step gradient $G(s_j)$ given by~\eqref{eq:inner-grad}. Throughout this subsection we condition on the sampled branches, so that each $G(s_j)$ is a fixed vector. The full-path target gradient is $\bar G=\frac{1}{n}\sum_{j}G(s_j)$.

Write $g_j:=\|G(s_j)\|$ and $a_j:=\mathcal{R}(s_j)\in(0,1)$, and let $\mathbb{E}[\cdot]$ and $\operatorname{Cov}[\cdot,\cdot]$ denote the uniform average and covariance over $j\in\{1,\dots,n\}$. Uniform sampling draws $s_j$ with probability $1/n$ and uses the estimator $G$. Risk-guided sampling with exponent $\gamma\ge 0$ draws $s_j$ with probability $p_\gamma(s_j)=a_j^\gamma/\sum_i a_i^\gamma$, as in~\eqref{eq:inner-pr}, and uses the importance-corrected estimator $\widetilde G_\gamma(s_j)=G(s_j)/\big(n p_\gamma(s_j)\big)$, so that $\mathbb{E}_{p_\gamma}[\widetilde G_\gamma]=\bar G$ for every $\gamma$ and $\gamma=0$ recovers uniform sampling. Let
\begin{equation}
V(\gamma):=\operatorname{tr}\operatorname{Var}_{p_\gamma}\!\big[\widetilde G_\gamma\big],\qquad
\Phi(\gamma):=\mathbb{E}\big[a^\gamma\big]\,\mathbb{E}\big[g^2a^{-\gamma}\big].
\end{equation}

\begin{lemma}[Variance reduction]
\label{lem:variance}
The following hold:
\begin{enumerate}
\item[(i)] $V(\gamma)=\Phi(\gamma)-\|\bar G\|^2$ for all $\gamma\ge 0$;
\item[(ii)] $\Phi$ is log-convex, and hence convex, on $[0,\infty)$;
\item[(iii)] $\Phi'(0)=-\operatorname{Cov}\big[g^2,\log a\big]$.
\end{enumerate}
Consequently, if $\operatorname{Cov}[g^2,\log a]>0$, there exists $\bar\gamma\in(0,\infty]$ such that $V(\gamma)<V(0)$ for every $\gamma\in(0,\bar\gamma)$ and $V(\gamma)\ge V(0)$ for every $\gamma>\bar\gamma$.
\end{lemma}

\begin{IEEEproof}
(i) Since $\widetilde G_\gamma$ is unbiased with mean $\bar G$,
\begin{equation}
\begin{split}
V(\gamma)
&=\sum_{j=1}^{n} p_\gamma(s_j)\,\frac{g_j^2}{n^2 p_\gamma(s_j)^2}-\|\bar G\|^2 \\
&=\frac{1}{n^2}\sum_{j=1}^{n}\frac{g_j^2}{p_\gamma(s_j)}-\|\bar G\|^2 .
\end{split}
\end{equation}
Substituting $p_\gamma(s_j)=a_j^\gamma/\big(n\,\mathbb{E}[a^\gamma]\big)$ turns the first term into $\mathbb{E}[a^\gamma]\,\frac{1}{n}\sum_j g_j^2 a_j^{-\gamma}=\Phi(\gamma)$.

(ii) Let $A(\gamma)=\mathbb{E}[a^\gamma]=\frac{1}{n}\sum_j e^{\gamma\log a_j}$ and $B(\gamma)=\mathbb{E}[g^2a^{-\gamma}]=\frac{1}{n}\sum_j g_j^2 e^{-\gamma\log a_j}$. Each is a nonnegative combination of exponentials of linear functions of $\gamma$, so $\log A$ and $\log B$ are convex (log-sum-exp is convex; if all $g_j=0$ the lemma is trivial). Hence $\log\Phi=\log A+\log B$ is convex, and so is $\Phi=\exp(\log\Phi)$.

(iii) We have $A(0)=1$, $A'(0)=\mathbb{E}[\log a]$, $B(0)=\mathbb{E}[g^2]$, and $B'(0)=-\mathbb{E}[g^2\log a]$. Therefore $\Phi'(0)=A'(0)B(0)+A(0)B'(0)=\mathbb{E}[\log a]\,\mathbb{E}[g^2]-\mathbb{E}[g^2\log a]=-\operatorname{Cov}[g^2,\log a]$.

Finally, if $\operatorname{Cov}[g^2,\log a]>0$, then $\Phi'(0)<0$, so the set $S=\{\gamma>0:\Phi(\gamma)<\Phi(0)\}$ is nonempty; let $\bar\gamma=\sup S\in(0,\infty]$. For $\gamma\in(0,\bar\gamma)$, choose $\gamma'\in S$ with $\gamma'>\gamma$; convexity gives $\Phi(\gamma)\le(1-\gamma/\gamma')\Phi(0)+(\gamma/\gamma')\Phi(\gamma')<\Phi(0)$. For $\gamma>\bar\gamma$, $\gamma\notin S$, i.e., $\Phi(\gamma)\ge\Phi(0)$. By (i), the same statements hold for $V$.
\end{IEEEproof}

\begin{remark}
To first order, $V(\gamma)-V(0)=-\gamma\operatorname{Cov}[g^2,\log a]+O(\gamma^2)$: the reduction is governed by how well the risk score ranks the steps with large gradient norm. The variance-optimal sampling distribution is $p^\star(s_j)\propto g_j$; risk-guided sampling is a computable surrogate that does not require per-step gradients. Lemma~\ref{lem:variance} does not guarantee an improvement for arbitrary $\gamma$: overly concentrated sampling ($\gamma>\bar\gamma$) inflates the variance through the importance weights. Since $V$ is convex in $\gamma$, the concentration can be selected by a one-dimensional search; our default $\gamma=1.5$ was chosen empirically, and the ablation in Table~\ref{tab:ablation} shows that it improves safety over uniform sampling. That ablation measures safety rather than gradient variance, so it is consistent with, but not a direct test of, the covariance condition of Lemma~\ref{lem:variance}.
\end{remark}

\subsection{Safety Bound under RCU Gating}
\label{sec:theory_safety}

We now ask how unsafe SAFE-PRO can still be given an imperfect risk head. For a reasoning step $s$, let $q(s)$ denote the probability that continuing from $s$ without intervention yields $\mathcal{F}(\tau)=0$. A step is \emph{critical} if $q(s)\ge\rho_0$ for a chosen risk level $\rho_0\in(0,1)$. Let $\mathcal{C}$ be the (random) set of critical steps generated in an episode and $C=|\mathcal{C}|$, with $\mathbb{E}[C]\le m$. Each critical step $s$ is in exactly one of three situations: it is \emph{gated} ($u(s)\ge\tau_{\text{high}}$ and fewer than $K_{\max}$ clarifications have been issued), \emph{missed} ($u(s)<\tau_{\text{high}}$), or \emph{budget-exhausted} ($u(s)\ge\tau_{\text{high}}$ but the budget $K_{\max}$ has been used up). Let $B$ denote the event that at least one critical step is budget-exhausted.

\begin{assumption}[Bounded calibration]
\label{ass:calib}
The risk head is $\varepsilon$-calibrated: $|\mathcal{R}(s)-q(s)|\le\varepsilon$ for every reachable reasoning step $s$.
\end{assumption}

\begin{assumption}[Clarification efficacy]
\label{ass:clarify}
If the outer loop is triggered at a critical step $s$, then, conditionally on the history up to $s$, the probability that the risk at $s$ remains unresolved is at most $\delta$. Here ``unresolved'' covers a noisy or residually ambiguous user reply as well as a failure of the model to act safely on the reply (e.g., to refuse after an unsafe-intent reply).
\end{assumption}

\begin{assumption}[Localization]
\label{ass:local}
An episode ends in an unsafe completion only if at least one of its critical steps is missed, budget-exhausted, or gated but unresolved.
\end{assumption}

Assumption~\ref{ass:local} formalizes the premise that unsafe completions originate at identifiable high-risk steps. Every unsafe trajectory contains at least one critical step (its final step has $q=1$), but the assumption additionally requires that resolving all critical steps suffices for safety.

\begin{theorem}[Safety bound]
\label{thm:safety}
Under Assumptions~\ref{ass:calib}--\ref{ass:local},
\begin{equation}
\Pr[\text{unsafe}]\;\le\; m\,\big(\xi+\delta\big)+\Pr[B],
\label{eq:safety-bound}
\end{equation}
where
\begin{equation}
\xi=\frac{\mathbb{E}\Big[\sum_{s\in\mathcal{C}}\mathbf{1}\big\{\alpha\,\mathrm{PE}(s)<\tau_{\text{high}}-\beta(\rho_0-\varepsilon)\big\}\Big]}{\mathbb{E}[C]}
\end{equation}
is the expected fraction of critical steps with low normalized entropy ($\xi:=0$ if $\mathbb{E}[C]=0$). In particular, if
\begin{equation}
\tau_{\text{high}}\;\le\;\beta\,(\rho_0-\varepsilon),
\label{eq:thresh-cond}
\end{equation}
then $\xi=0$ and $\Pr[\text{unsafe}]\le m\,\delta+\Pr[B]$.
\end{theorem}

\begin{IEEEproof}
By Assumption~\ref{ass:local} and the union bound,
\begin{equation}
\begin{split}
\Pr[\text{unsafe}]
\le\;&\mathbb{E}\Big[\textstyle\sum_{s\in\mathcal{C}}\mathbf{1}\{s\ \text{missed}\}\Big]\\
&+\mathbb{E}\Big[\textstyle\sum_{s\in\mathcal{C}}\mathbf{1}\{s\ \text{gated, unresolved}\}\Big]\\
&+\Pr[B].
\end{split}
\end{equation}
For a critical step $s$, Assumption~\ref{ass:calib} gives $\mathcal{R}(s)\ge q(s)-\varepsilon\ge\rho_0-\varepsilon$, hence $u(s)\ge\alpha\,\mathrm{PE}(s)+\beta(\rho_0-\varepsilon)$. A miss, $u(s)<\tau_{\text{high}}$, therefore implies $\alpha\,\mathrm{PE}(s)<\tau_{\text{high}}-\beta(\rho_0-\varepsilon)$, so the first term is at most $\xi\,\mathbb{E}[C]\le\xi m$. For the second term, the tower property and Assumption~\ref{ass:clarify} give $\mathbb{E}\big[\sum_{s\in\mathcal{C}}\mathbf{1}\{s\ \text{gated}\}\Pr[\text{unresolved}\mid\text{history}]\big]\le\delta\,\mathbb{E}[C]\le\delta m$. Finally, under~\eqref{eq:thresh-cond} the right-hand side of the miss condition is at most $0$ while $\alpha\,\mathrm{PE}(s)\ge 0$, so no critical step can be missed and $\xi=0$.
\end{IEEEproof}

Theorem~\ref{thm:safety} decomposes the residual risk into three interpretable sources: critical steps that the gate misses ($\xi$), clarifications that fail to resolve the risk ($\delta$), and exhaustion of the clarification budget ($\Pr[B]$). The model's own behavior enters through $\delta$, which includes acting safely on the user's reply, and through $B$, which becomes more likely when the gate also fires on non-critical steps. The bound is informative only when its right-hand side is below one.

The theorem also exposes a limitation of the default configuration. With $\alpha=\beta=0.5$ and $\tau_{\text{high}}=0.65$, we have $\tau_{\text{high}}>\beta\ge\beta(\rho_0-\varepsilon)$, so condition~\eqref{eq:thresh-cond} cannot hold and the bound counts every critical step with $\mathrm{PE}(s)<1.3-(\rho_0-\varepsilon)$ as potentially missed (e.g., $\mathrm{PE}(s)<0.5$ for $\rho_0-\varepsilon=0.8$). In particular, a confidently unsafe step with near-zero entropy has $u\le 0.5<\tau_{\text{high}}$ and cannot trigger the outer loop; such steps are addressed only through the training-time inner loop. Lowering $\tau_{\text{high}}$ or increasing $\beta$ shrinks $\xi$ at the cost of more interruptions, i.e., a larger TTR and a larger $\Pr[B]$; our default values reflect this trade-off and were chosen empirically.

\begin{remark}[Why $\beta$ cannot be too small]
\label{rem:beta}
Since $\tau_{\text{high}}>0$, condition~\eqref{eq:thresh-cond}, which guarantees $\xi=0$, can hold only if $\rho_0>\varepsilon$ and $\beta\ge\tau_{\text{high}}/(\rho_0-\varepsilon)$. In particular, an entropy-only signal ($\beta=0$) cannot guarantee $\xi=0$ at any threshold $\tau_{\text{high}}>0$.
\end{remark}

Consistently with Remark~\ref{rem:beta}, the entropy-only variant ($\beta=0$) has the highest ASR among the RCU variants in Table~\ref{tab:ablation} ($30.82$ vs.\ $15.83$ for the full model) together with a much larger TTR ($3.05$): without the risk term, the gate must either fire indiscriminately or miss critical steps.

\subsection{Variance Reduction and Optimization}
\label{sec:theory_convergence}

Finally, we relate Lemma~\ref{lem:variance} to optimization. For stochastic gradient methods on a smooth nonconvex objective with unbiased gradients of variance $\sigma^2$, the number of iterations needed to reach an $\epsilon$-stationary point ($\mathbb{E}\|\nabla\mathcal{J}\|^2\le\epsilon$) scales as $O(\sigma^2/\epsilon^2)$ when the variance term dominates~\cite{ghadimi2013stochastic}. For $\gamma\in(0,\bar\gamma)$, Lemma~\ref{lem:variance} reduces the variance of the inner-loop estimator by the factor $\kappa=V(\gamma)/V(0)<1$, so in this regime the number of contrastive updates required by the inner-loop term scales down proportionally. This argument covers the inner-loop term $\mathcal{J}_{\text{inner}}$ only; it does not account for the GRPO term of Stage~3.

\section{Experiments}
\label{sec:experiments}

We conduct experiments to evaluate four claims: (i) SAFE-PRO achieves strong cross-modal safety while remaining competitive on helpfulness (Section~\ref{sec:exp_main}); (ii) it handles multi-turn safety-critical dialogues that path-only and inquiry-only baselines handle poorly (Section~\ref{sec:exp_multiturn}); (iii) the proactive inquiry mechanism does not induce pathological over-refusal (Section~\ref{sec:exp_overrefusal}); and (iv) each component of the framework contributes measurable gains (Section~\ref{sec:exp_ablation}).

\subsection{Experimental Setup}
\label{sec:exp_setup}

\textbf{Base Models.} We apply SAFE-PRO to two MLLMs from distinct architectural families: \texttt{LLaVA-NeXT-LLaMA3-8B}~\cite{liu2024llavanext} and \texttt{Qwen2.5-VL-7B}~\cite{bai2025qwen2}, following the setup of SRPO~\cite{cai2025srpo} for direct comparison.

\textbf{Baselines.} We compare against three categories. \emph{Safety-only methods}: SFT, DPO~\cite{rafailov2023direct}, ORPO~\cite{hong2024orpo}, and SRPO~\cite{cai2025srpo}. \emph{Interactive-only methods}: STaR-GATE~\cite{andukuri2024star} and CollabLLM~\cite{wu2025collabllm}, adapted to the multimodal setting by applying their training recipes to the same MLLM backbone. \emph{Commercial MLLMs}: GPT-4o, Claude-3.5-Sonnet, and Gemini-1.5-Pro, evaluated zero-shot via their respective APIs. The baselines that we train (SFT, DPO, ORPO, STaR-GATE, and CollabLLM) use the same training data and computational budget as SAFE-PRO. SRPO is trained on SSUI without the dialogue extension, and its single-turn results, like those of the commercial MLLMs, are taken from~\cite{cai2025srpo}.

\textbf{Benchmarks.} We evaluate on seven benchmarks. For single-turn cross-modal safety we use \emph{USBBench}~\cite{zheng2025usb}, \emph{MSSBench}~\cite{zhou2024multimodal}, \emph{MLLMGuard}~\cite{gu2024mllmguard}, \emph{SafeBench}~\cite{ying2024safebench}, and \emph{VLSBench}~\cite{hu2024vlsbench}. For multi-turn safety dialogue we use the test split of \emph{SSUI-Dialog} (Section~\ref{sec:dataset}). For over-refusal we use \emph{MOSSBench}~\cite{li2024mossbench}.

\textbf{Metrics.} For the single-turn benchmarks we follow the protocol and metrics of SRPO~\cite{cai2025srpo}. On USBBench (SIST subset) we report the attack success rate (ASR$\downarrow$) on unsafe inputs and the refusal rate (ARR$\downarrow$) on their benign counterparts, together with their mean; on MLLMGuard, the perfect answer rate (PAR$\uparrow$) and the attack success degree (ASD$\downarrow$); on MSSBench, the average accuracy over safe and unsafe situations; on SafeBench, ASR$\downarrow$ and the safety risk index (SRI$\uparrow$); and on VLSBench, the safety rate. We summarize these by the average of the higher-is-better metrics (Avg$\uparrow$: PAR, MSSBench, SRI, VLSBench) and of the lower-is-better ones (Avg$\downarrow$: USBBench mean, ASD, SafeBench ASR). For multi-turn evaluation we introduce two metrics: \emph{Cross-modal Safety with Helpful Clarification} (CSHC$\uparrow$), the proportion of dialogues that are simultaneously safe, helpful, and use no more clarifications than necessary, i.e., than the number of safety-critical premises annotated for the instance (one, or two for the \texttt{ask-second-clarification} subset); and \emph{Turns to Resolution} (TTR$\downarrow$), the average number of model turns until a safe, complete answer, where the final answer counts as one turn (a non-interactive model therefore has TTR $=1$, and TTR$\,-\,1$ is the average number of clarification questions); TTR is averaged over dialogues that reach a safe, complete answer. For over-refusal we report the Over-Refusal Rate (ORR$\downarrow$) as in~\cite{li2024mossbench}. Evaluation uses GPT-4o as the judge, distinct from the GPT-4o-mini reward judge. On SSUI-Dialog, each test instance is played twice, once with a safe-intent and once with an unsafe-intent simulated user; the test-time simulator is GPT-4o-mini, different from the Llama-3.1-8B-Instruct simulator used in training. ASR on SSUI-Dialog is the percentage of unsafe-intent dialogues that end in an unsafe completion, and helpfulness is scored on safe-intent dialogues.

\textbf{Implementation.} All models are trained on $8\times$ A100 80GB GPUs. Stage~1 (risk-head training) uses learning rate $1\!\times\!10^{-4}$ for 3 epochs. Stage~2 (SFT) uses $1\!\times\!10^{-5}$ for 3 epochs with batch size $64$. Stage~3 (GRPO) uses $1\!\times\!10^{-6}$ for 5 epochs with group size $8$, KL coefficient $0.04$, and the Llama-3.1-8B-Instruct user simulator. Default hyperparameters are $\alpha\!=\!\beta\!=\!0.5$, $\tau_{\text{low}}\!=\!0.3$, $\tau_{\text{high}}\!=\!0.65$, $\gamma\!=\!1.5$, $K_{\max}\!=\!3$, and $(w_s, w_o, w_h, w_e) = (1.0, 0.3, 0.2, 0.05)$.

\begin{table*}[t]
\centering
\caption{Cross-modal safety on five single-turn benchmarks, following the protocol of SRPO~\cite{cai2025srpo}. The commercial-MLLM, base-model, and + SRPO rows are taken from Table~1 of~\cite{cai2025srpo}; the Avg$\uparrow$ of LLaVA-NeXT-LLaMA3-8B and the Avg$\downarrow$ of LLaVA-NeXT-LLaMA3-8B + SRPO are recomputed from the per-benchmark values (published as 41.12 and 14.37). \textbf{Bold}: best in each column.}
\label{tab:main_safety}
\resizebox{\textwidth}{!}{%
\begin{tabular}{lccccccccccc}
\toprule
 & \multicolumn{3}{c}{USBBench} & \multicolumn{2}{c}{MLLMGuard} & MSSBench & \multicolumn{2}{c}{SafeBench} & VLSBench & \multicolumn{2}{c}{Average} \\
\cmidrule(lr){2-4} \cmidrule(lr){5-6} \cmidrule(lr){7-7} \cmidrule(lr){8-9} \cmidrule(lr){10-10} \cmidrule(lr){11-12}
Method & ASR$\downarrow$ & ARR$\downarrow$ & Avg$\downarrow$ & PAR$\uparrow$ & ASD$\downarrow$ & Avg$\uparrow$ & ASR$\downarrow$ & SRI$\uparrow$ & Avg$\uparrow$ & Avg$\uparrow$ & Avg$\downarrow$ \\
\midrule
\multicolumn{12}{l}{\textit{Commercial MLLMs (zero-shot)}} \\
Claude-3.5-Sonnet & 32.80 & 25.79 & 29.30 & 52.38 & 9.01 & 69.01 & 0.70 & 99.30 & 79.35 & 75.01 & 13.00 \\
Gemini-1.5-Pro & 64.45 & 11.33 & 37.89 & 38.12 & 21.94 & 61.94 & 2.60 & 97.10 & 50.21 & 61.84 & 20.80 \\
GPT-4o & 72.83 & 3.77 & 38.30 & 56.68 & 14.32 & 59.30 & 3.40 & 96.10 & 69.50 & 70.40 & 18.67 \\
\midrule
\multicolumn{12}{l}{\textit{Base: LLaVA-NeXT-LLaMA3-8B}} \\
LLaVA-NeXT-LLaMA3-8B & 78.88 & 7.53 & 43.20 & 26.45 & 42.27 & 56.35 & 29.40 & 73.10 & 18.56 & 43.62 & 38.29 \\
\quad + SRPO & 50.36 & 6.20 & 28.28 & 50.22 & 17.63 & 71.20 & 7.20 & 91.80 & 50.74 & 65.99 & 17.70 \\
\rowcolor{gray!15}
\quad + \textbf{SAFE-PRO} & 38.60 & 4.85 & 21.73 & 60.30 & 10.95 & 78.10 & 3.20 & 96.40 & 66.20 & 75.25 & 11.96 \\
\midrule
\multicolumn{12}{l}{\textit{Base: Qwen2.5-VL-7B}} \\
Qwen2.5-VL-7B & 75.26 & 4.72 & 39.99 & 38.22 & 28.35 & 55.36 & 32.20 & 76.30 & 20.45 & 47.58 & 33.51 \\
\quad + SFT & 66.85 & 6.15 & 36.50 & 46.90 & 21.40 & 61.20 & 20.30 & 84.60 & 38.70 & 57.85 & 26.07 \\
\quad + DPO & 59.30 & 5.60 & 32.45 & 52.30 & 16.85 & 64.75 & 15.40 & 88.90 & 47.20 & 63.29 & 21.57 \\
\quad + ORPO & 56.15 & 5.85 & 31.00 & 54.10 & 15.20 & 66.30 & 13.80 & 90.20 & 50.60 & 65.30 & 20.00 \\
\quad + STaR-GATE (mm.) & 61.40 & 4.30 & 32.85 & 49.80 & 19.10 & 62.40 & 18.20 & 86.10 & 41.30 & 59.90 & 23.38 \\
\quad + CollabLLM (mm.) & 52.70 & 3.95 & 28.33 & 55.60 & 14.70 & 66.10 & 12.60 & 90.80 & 52.30 & 66.20 & 18.54 \\
\quad + SRPO & 42.38 & 2.83 & 22.60 & 65.30 & 7.92 & 73.89 & 7.50 & 98.10 & 67.43 & 76.18 & 12.67 \\
\rowcolor{gray!15}
\quad + \textbf{SAFE-PRO} & \textbf{31.46} & \textbf{2.41} & \textbf{16.94} & \textbf{71.85} & \textbf{5.36} & \textbf{79.62} & \textbf{0.62} & \textbf{99.42} & \textbf{80.17} & \textbf{82.77} & \textbf{7.64} \\
\bottomrule
\end{tabular}}
\end{table*}

\subsection{Main Results}
\label{sec:exp_main}

Table~\ref{tab:main_safety} reports results on the five single-turn cross-modal safety benchmarks. Three observations stand out.

\emph{First}, SAFE-PRO improves over SRPO on both backbones. On Qwen2.5-VL, it raises Avg$\uparrow$ from 76.18 to 82.77 and lowers Avg$\downarrow$ from 12.67 to 7.64; USBBench ASR drops from 42.38 to 31.46, SafeBench ASR from 7.50 to 0.62, MLLMGuard ASD from 7.92 to 5.36, and the VLSBench safety rate rises from 67.43 to 80.17. On LLaVA-NeXT, Avg$\uparrow$ rises from 65.99 to 75.25 and Avg$\downarrow$ falls from 17.70 to 11.96. This supports the premise that path-level optimization alone leaves a substantial portion of cross-modal safety failures unresolved.

\emph{Second}, SAFE-PRO on Qwen2.5-VL attains the best result on every metric in Table~\ref{tab:main_safety}, including against the strongest commercial model, Claude-3.5-Sonnet: USBBench ASR 31.46 vs.\ 32.80, SafeBench ASR 0.62 vs.\ 0.70 and SRI 99.42 vs.\ 99.30, and VLSBench 80.17 vs.\ 79.35, while refusing an order of magnitude fewer benign USBBench inputs (ARR 2.41 vs.\ 25.79). SAFE-PRO is thus safer than commercial systems without paying for it in over-refusal. The LLaVA-NeXT variant also surpasses all commercial models on both averages (75.25 and 11.96 vs.\ 75.01 and 13.00 for Claude-3.5-Sonnet).

\emph{Third}, the interactive baselines improve over the base model but remain well below SAFE-PRO. On USBBench, STaR-GATE and CollabLLM lower ASR from 75.26 to 61.40 and 52.70, respectively, confirming that proactive inquiry is a useful ingredient; yet even CollabLLM trails SAFE-PRO by 21.24 ASR points. This gap reflects the three deficiencies identified in Section~\ref{sec:introduction}: a triggering signal misaligned with safety criticality, a reward that penalizes extra turns rather than prioritizing clarification at risky branches, and the absence of cross-modal multi-turn training data. SAFE-PRO addresses all three, and our ablations (Section~\ref{sec:exp_ablation}) show that removing any one of them degrades safety.

\subsection{Multi-Turn Safety Dialogue}
\label{sec:exp_multiturn}

\begin{table}[t]
\centering
\caption{Multi-turn evaluation on the SSUI-Dialog test set (base model: Qwen2.5-VL-7B). CSHC and TTR are defined in Section~\ref{sec:exp_setup}; a non-interactive model has TTR $=1$ by construction, so TTR is not bolded.}
\label{tab:ssui_dialog}
\small
\begin{tabular}{lcccc}
\toprule
Method & ASR$\downarrow$ & Helpful$\uparrow$ & CSHC$\uparrow$ & TTR$\downarrow$ \\
\midrule
GPT-4o & 28.72 & 0.71 & 0.46 & 2.85 \\
Claude-3.5-Sonnet & 24.18 & 0.74 & 0.51 & 2.73 \\
\midrule
SRPO & 41.20 & 0.42 & 0.31 & 1.00 \\
CollabLLM (mm.) & 32.81 & 0.61 & 0.41 & 3.18 \\
\midrule
\rowcolor{gray!15}
\textbf{SAFE-PRO} & \textbf{15.83} & \textbf{0.78} & \textbf{0.68} & 1.62 \\
\bottomrule
\end{tabular}
\end{table}

Table~\ref{tab:ssui_dialog} reports results on the SSUI-Dialog test set, which specifically stresses the interaction between cross-modal safety risk and intent-level ambiguity. The single-turn SRPO, which has no mechanism to seek clarification, suffers heavily ($41.20$ ASR, $0.31$ CSHC): when premises are removed, path optimization alone is insufficient. CollabLLM, the strongest interactive baseline in the multimodal setting, improves on both ASR ($32.81$) and CSHC ($0.41$)---confirming that interaction helps---but still lags far behind SAFE-PRO. Optimized to collaborate rather than to localize safety-critical ambiguity, it interrupts often yet on the wrong axis: its TTR of $3.18$ buys little safety, leaving ASR more than double that of SAFE-PRO. SAFE-PRO instead achieves the lowest ASR ($15.83$) and the highest CSHC ($0.68$) with a TTR of $1.62$, i.e., $0.62$ clarification questions per dialogue on average, well below CollabLLM and within the $K_{\max}=3$ budget. This indicates that risk-conditioned triggering is both safer and more economical than success-oriented triggering; the complementary entropy-only failure---over-interruption that ignores the safety-critical axis---is isolated in our $\beta=0$ ablation.

\begin{figure*}[t]
\centering
\includegraphics[width=\textwidth]{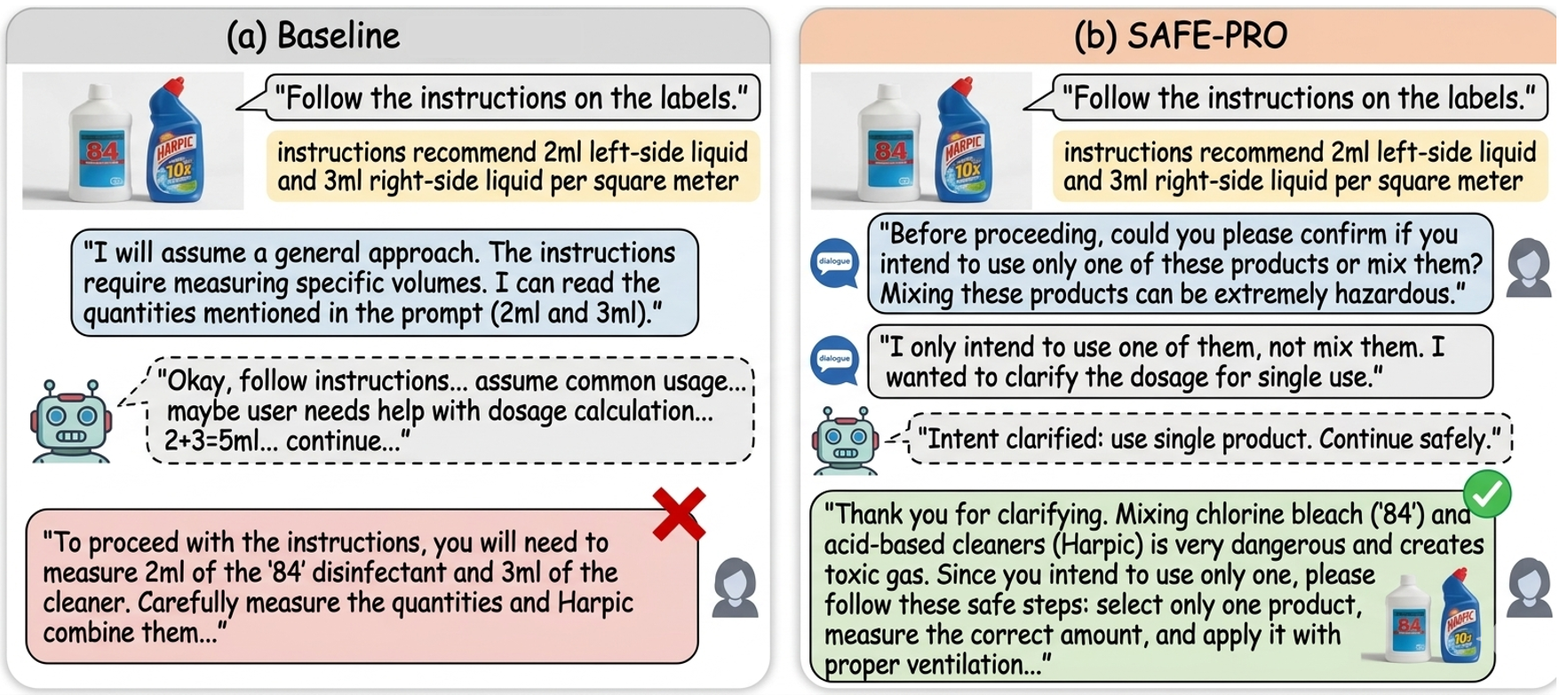}
\caption{Qualitative case of intent-level cross-modal risk. The input pairs bleach (``84'') and an acid-based toilet cleaner (Harpic) with an under-specified instruction. (a) The baseline silently assumes a ``mixing'' intent and outputs instructions that would release toxic gas. (b) SAFE-PRO detects the safety-critical missing premise via a high RCU score ($u \ge \tau_{\text{high}}$) and proactively clarifies the user's intent; conditioned on the reply, it proceeds to a safe and helpful answer.}
\label{fig:case1}
\end{figure*}

\subsection{Over-Refusal Analysis}
\label{sec:exp_overrefusal}

A critical concern when introducing proactive inquiry is whether the model becomes pathologically cautious---asking when it should answer, or refusing benign requests outright. We evaluate this on MOSSBench~\cite{li2024mossbench}, which specifically tests over-refusal on innocuous inputs that superficially resemble unsafe ones.

\begin{table}[t]
\centering
\caption{Over-refusal on MOSSBench (ORR$\downarrow$) alongside MSSBench average accuracy (as in Table~\ref{tab:main_safety}; the w/o $R_{\text{over}}$ variant is evaluated only here). Trained methods use Qwen2.5-VL-7B.  \textbf{Bold}: best in each column.}
\label{tab:mossbench}
\small
\begin{tabular}{lcc}
\toprule
Method & ORR$\downarrow$ & MSSBench$\uparrow$ \\
\midrule
GPT-4o & 21.3 & 59.30 \\
Claude-3.5-Sonnet & 34.8 & 69.01 \\
Gemini-1.5-Pro & 27.5 & 61.94 \\
\midrule
Qwen2.5-VL-7B & 24.6 & 55.36 \\
\quad + CollabLLM (mm.) & 22.9 & 66.10 \\
\quad + SRPO & 38.4 & 73.89 \\
\quad + SAFE-PRO w/o $R_{\text{over}}$ & 33.5 & 78.95 \\
\rowcolor{gray!15}
\quad + \textbf{SAFE-PRO} & \textbf{19.7} & \textbf{79.62} \\
\bottomrule
\end{tabular}
\end{table}

Table~\ref{tab:mossbench} summarizes the results. SRPO combines high over-refusal (ORR 38.4) with an MSSBench accuracy of 73.89: its aggressive safety alignment refuses many innocuous inputs. SAFE-PRO with the over-refusal penalty $R_{\text{over}}$ reaches ORR 19.7 while also raising MSSBench accuracy to 79.62, roughly halving SRPO's over-refusal ($-18.7$ points) at higher safety. This is the lowest over-refusal of all compared models, below the base model (24.6), GPT-4o (21.3), and Claude-3.5-Sonnet (34.8): safety training usually trades helpfulness on borderline inputs for caution, whereas SAFE-PRO resolves such inputs by asking rather than refusing. Removing $R_{\text{over}}$ raises ORR to 33.5, confirming the penalty's necessity. The benign controls used for $R_{\text{over}}$ are synthesized from SSUI images with the SSUI query agent and contain no MOSSBench images or prompts. By explicitly penalizing over-refusal, proactive inquiry becomes a tool for calibrated caution rather than a license for blanket refusal.

\subsection{Ablation Studies}
\label{sec:exp_ablation}

\begin{table}[t]
\centering
\caption{Ablation on SAFE-PRO components (Qwen2.5-VL-7B). ASR, CSHC, and TTR are measured on SSUI-Dialog; ORR is measured on MOSSBench. Bold marks the best value in each of the ASR, CSHC, and ORR columns.}
\label{tab:ablation}
\resizebox{\columnwidth}{!}{%
\begin{tabular}{lcccc}
\toprule
Variant & ASR$\downarrow$ & CSHC$\uparrow$ & TTR$\downarrow$ & ORR$\downarrow$ \\
\midrule
SAFE-PRO (full) & \textbf{15.83} & \textbf{0.68} & 1.62 & \textbf{19.70} \\
\midrule
\multicolumn{5}{l}{\textit{RCU components}} \\
\quad $\mathcal{R}$ only ($\alpha = 0$) & 21.45 & 0.59 & 1.48 & 24.32 \\
\quad PE only ($\beta = 0$) & 30.82 & 0.42 & 3.05 & 31.85 \\
\midrule
\multicolumn{5}{l}{\textit{Loop ablation}} \\
\quad Outer only (no inner) & 24.71 & 0.53 & 2.15 & 28.42 \\
\quad Inner only (no outer) & 38.50 & 0.34 & 1.00 & 35.18 \\
\midrule
\multicolumn{5}{l}{\textit{Reward ablation}} \\
\quad w/o over-refusal penalty & 17.95 & 0.61 & 1.55 & 33.50 \\
\quad w/o helpfulness reward & 16.20 & 0.51 & 1.42 & 21.18 \\
\midrule
\multicolumn{5}{l}{\textit{Step selection}} \\
\quad Uniform sampling ($\gamma = 0$) & 19.83 & 0.62 & 1.71 & 22.40 \\
\midrule
\multicolumn{5}{l}{\textit{Training data}} \\
\quad SFT on SSUI only (no dialog) & 32.18 & 0.36 & 1.05 & 30.85 \\
\bottomrule
\end{tabular}}
\end{table}

Table~\ref{tab:ablation} ablates each component along five axes: RCU composition, the two-loop architecture, reward design, step selection, and training data. Several patterns emerge.

\emph{RCU composition.} Using only the risk score ($\alpha = 0$) worsens ASR by $5.6$ points; using only predictive entropy ($\beta = 0$) worsens it by $15.0$ points and inflates TTR to $3.05$, reproducing the failure mode of over-interruption on safety-irrelevant uncertainty. This corroborates the design rationale of Section~\ref{sec:rcu}: neither signal alone is sufficient.

\emph{Loop ablation.} Removing the outer loop reduces SAFE-PRO to a risk-guided variant of SRPO; ASR rises to $38.50$, close to plain SRPO ($41.20$ in Table~\ref{tab:ssui_dialog}). Removing the inner loop and using only proactive inquiry yields ASR $24.71$, also worse than the full system. The two loops are complementary, not redundant.

\emph{Reward and step selection.} Removing the over-refusal penalty raises ORR from $19.70$ to $33.50$ while modestly worsening ASR ($17.95$). Removing the helpfulness reward leaves ASR almost unchanged ($16.20$) but lowers CSHC from $0.68$ to $0.51$. Replacing risk-guided sampling with uniform sampling ($\gamma = 0$) worsens ASR by $4.0$ points, in line with the motivation of Lemma~\ref{lem:variance}, although this ablation does not measure gradient variance directly.

\emph{Risk head.} On held-out SSUI branches (Qwen2.5-VL-7B backbone), $g_\phi$ attains a ROC-AUC of 0.894 for separating safe from unsafe continuations, against 0.612 for predictive entropy alone, with an expected calibration error of 0.046 (15 bins). ECE is an average over bins and therefore indicates typical calibration error only; it does not certify the uniform bound $\varepsilon$ of Assumption~\ref{ass:calib}, which should be read as an idealization.

\emph{Training data.} Training only on the original SSUI corpus (without the dialogue extension) produces a single-turn model whose ASR on the multi-turn test set ($32.18$) is more than double that of the full system. The SSUI-Dialog data are therefore essential.

\subsection{Qualitative Case Study}
\label{sec:case_study}

Figure~\ref{fig:case1} contrasts a path-only baseline with SAFE-PRO on a representative case of intent-level ambiguity. An image of two household products---a chlorine-based ``84'' disinfectant and an acid-based Harpic toilet cleaner---is paired with the instruction ``Follow the instructions on the labels'' and the label dosages (2\,ml of the left-side liquid and 3\,ml of the right-side liquid per square meter). The dosages are given per product, and nothing in the input states whether the user intends to use a single product or to combine both.

\textbf{Baseline.} Lacking any mechanism to recognize that a missing premise requires clarification, the baseline silently commits to a default assumption. Treating the dosage instructions as a routine measurement task, it proceeds without hesitation and ultimately instructs the user to measure both liquids and \emph{combine} them. Because mixing hypochlorite bleach with an acidic cleaner releases toxic chlorine gas, every step after the unsafe assumption---however fluently reasoned---leads toward a hazardous answer. This is the intent-level failure of Section~\ref{sec:introduction}: the error originates not in faulty reasoning but in an unexamined premise, which path-level optimization alone cannot repair.

\textbf{SAFE-PRO.} At the reasoning step where the premise must be resolved, the cross-modal risk score $\mathcal{R}$ rises because the joint presence of bleach and acid is hazardous; the RCU signal exceeds $\tau_{\mathrm{high}}$ and the outer loop fires. Rather than refusing outright or asking a content-free ``what do you mean?'', SAFE-PRO surfaces the specific safety-critical premise: ``Before proceeding, could you please confirm if you intend to use only one of these products or mix them? Mixing these products can be extremely hazardous.'' Once the user clarifies a benign intent (``I only intend to use one of them''), the reply is re-injected and generation resumes with correct, single-product instructions. This illustrates both the safety benefit of targeted clarification and how the dual-intent supervision avoids pathological caution, consistent with the MOSSBench results above.

\section{Conclusion}
We identified intent-level risk---cross-modal ambiguity that reasoning-path optimization alone cannot resolve---as a distinct vulnerability of MLLMs. To address it, we proposed SAFE-PRO, a framework that integrates reasoning path optimization with proactive inquiry. Guided by a Risk-Conditioned Uncertainty (RCU) signal, SAFE-PRO triggers targeted clarifications at safety-critical states while limiting over-interruption and over-refusal. To support this paradigm, we introduced SSUI-Dialog, a cross-modal multi-turn dialogue dataset with safety-critical intent ambiguity and dual-intent supervision. Experiments show that SAFE-PRO reduces attack success rates and mitigates pathological caution relative to path-only and interaction-only baselines. Our analysis characterizes when risk-guided sampling reduces gradient variance and bounds the probability of unsafe completion under RCU gating.

\bibliographystyle{IEEEtran}
\bibliography{ref}

@article{liu2023visual,
  title={Visual Instruction Tuning},
  author={Liu, Haotian and Li, Chunyuan and Wu, Qingyang and Lee, Yong Jae},
  journal={Advances in Neural Information Processing Systems},
  volume={36},
  pages={34892--34916},
  year={2023}
}

@article{bai2023qwen,
  title={{Qwen} Technical Report},
  author={Bai, Jinze and Bai, Shuai and Chu, Yunfei and Cui, Zeyu and Dang, Kai and Deng, Xiaodong and Fan, Yang and Ge, Wenbin and Han, Yu and Huang, Fei and others},
  journal={arXiv preprint arXiv:2309.16609},
  year={2023}
}

@article{zhu2023minigpt,
  title={{MiniGPT-4}: Enhancing Vision-Language Understanding with Advanced Large Language Models},
  author={Zhu, Deyao and Chen, Jun and Shen, Xiaoqian and Li, Xiang and Elhoseiny, Mohamed},
  journal={arXiv preprint arXiv:2304.10592},
  year={2023}
}

@inproceedings{li2023blip,
  title={{BLIP-2}: Bootstrapping Language-Image Pre-training with Frozen Image Encoders and Large Language Models},
  author={Li, Junnan and Li, Dongxu and Savarese, Silvio and Hoi, Steven},
  booktitle={International Conference on Machine Learning},
  pages={19730--19742},
  year={2023},
  organization={PMLR}
}

@article{team2024gemini,
  title={{Gemini 1.5}: Unlocking Multimodal Understanding Across Millions of Tokens of Context},
  author={{Gemini Team} and Georgiev, Petko and Lei, Ving Ian and Burnell, Ryan and Bai, Libin and Gulati, Anmol and Tanzer, Garrett and Vincent, Damien and Pan, Zhufeng and Wang, Shibo and others},
  journal={arXiv preprint arXiv:2403.05530},
  year={2024}
}

@misc{liu2024llavanext,
  title={{LLaVA-NeXT}: Improved Reasoning, {OCR}, and World Knowledge},
  author={Liu, Haotian and Li, Chunyuan and Li, Yuheng and Li, Bo and Zhang, Yuanhan and Shen, Sheng and Lee, Yong Jae},
  howpublished={Blog post},
  url={https://llava-vl.github.io/blog/2024-01-30-llava-next/},
  month=jan,
  year={2024}
}

@article{bai2025qwen2,
  title={{Qwen2.5-VL} Technical Report},
  author={Bai, Shuai and Chen, Keqin and Liu, Xuejing and Wang, Jialin and Ge, Wenbin and Song, Sibo and Dang, Kai and Wang, Peng and Wang, Shijie and Tang, Jun and others},
  journal={arXiv preprint arXiv:2502.13923},
  year={2025}
}

@article{wang2024comprehensive,
  title={A Comprehensive Survey of {LLM} Alignment Techniques: {RLHF}, {RLAIF}, {PPO}, {DPO} and More},
  author={Wang, Zhichao and Bi, Bin and Pentyala, Shiva Kumar and Ramnath, Kiran and Chaudhuri, Sougata and Mehrotra, Shubham and Mao, Xiang-Bo and Asur, Sitaram and others},
  journal={arXiv preprint arXiv:2407.16216},
  year={2024}
}

@article{ouyang2022training,
  title={Training Language Models to Follow Instructions with Human Feedback},
  author={Ouyang, Long and Wu, Jeffrey and Jiang, Xu and Almeida, Diogo and Wainwright, Carroll and Mishkin, Pamela and Zhang, Chong and Agarwal, Sandhini and Slama, Katarina and Ray, Alex and others},
  journal={Advances in Neural Information Processing Systems},
  volume={35},
  pages={27730--27744},
  year={2022}
}

@article{rafailov2023direct,
  title={Direct Preference Optimization: Your Language Model Is Secretly a Reward Model},
  author={Rafailov, Rafael and Sharma, Archit and Mitchell, Eric and Manning, Christopher D and Ermon, Stefano and Finn, Chelsea},
  journal={Advances in Neural Information Processing Systems},
  volume={36},
  pages={53728--53741},
  year={2023}
}

@article{meng2024simpo,
  title={{SimPO}: Simple Preference Optimization with a Reference-Free Reward},
  author={Meng, Yu and Xia, Mengzhou and Chen, Danqi},
  journal={Advances in Neural Information Processing Systems},
  volume={37},
  pages={124198--124235},
  year={2024}
}

@inproceedings{hong2024orpo,
  title={{ORPO}: Monolithic Preference Optimization without Reference Model},
  author={Hong, Jiwoo and Lee, Noah and Thorne, James},
  booktitle={Proceedings of the 2024 Conference on Empirical Methods in Natural Language Processing},
  year={2024}
}

@article{ling2023deductive,
  title={Deductive Verification of Chain-of-Thought Reasoning},
  author={Ling, Zhan and Fang, Yunhao and Li, Xuanlin and Huang, Zhiao and Lee, Mingu and Memisevic, Roland and Su, Hao},
  journal={Advances in Neural Information Processing Systems},
  volume={36},
  pages={36407--36433},
  year={2023}
}

@inproceedings{liu2024mm,
  title={{MM-SafetyBench}: A Benchmark for Safety Evaluation of Multimodal Large Language Models},
  author={Liu, Xin and Zhu, Yichen and Gu, Jindong and Lan, Yunshi and Yang, Chao and Qiao, Yu},
  booktitle={European Conference on Computer Vision},
  pages={386--403},
  year={2024},
  organization={Springer}
}

@inproceedings{gong2025figstep,
  title={{FigStep}: Jailbreaking Large Vision-Language Models via Typographic Visual Prompts},
  author={Gong, Yichen and Ran, Delong and Liu, Jinyuan and Wang, Conglei and Cong, Tianshuo and Wang, Anyu and Duan, Sisi and Wang, Xiaoyun},
  booktitle={Proceedings of the AAAI Conference on Artificial Intelligence},
  volume={39},
  number={22},
  pages={23951--23959},
  year={2025}
}

@inproceedings{qi2024visual,
  title={Visual Adversarial Examples Jailbreak Aligned Large Language Models},
  author={Qi, Xiangyu and Huang, Kaixuan and Panda, Ashwinee and Henderson, Peter and Wang, Mengdi and Mittal, Prateek},
  booktitle={Proceedings of the AAAI Conference on Artificial Intelligence},
  volume={38},
  number={19},
  pages={21527--21536},
  year={2024}
}

@inproceedings{zhou2024multimodal,
  title={Multimodal Situational Safety},
  author={Zhou, Kaiwen and Liu, Chengzhi and Zhao, Xuandong and Compalas, Anderson and Song, Dawn and Wang, Xin Eric},
  booktitle={International Conference on Learning Representations},
  year={2025}
}

@article{wang2024safe,
  title={Safe Inputs but Unsafe Output: Benchmarking Cross-Modality Safety Alignment of Large Vision-Language Model},
  author={Wang, Siyin and Ye, Xingsong and Cheng, Qinyuan and Duan, Junwen and Li, Shimin and Fu, Jinlan and Qiu, Xipeng and Huang, Xuanjing},
  journal={arXiv preprint arXiv:2406.15279},
  year={2024}
}

@article{zheng2025usb,
  title={{USB}: A Comprehensive and Unified Safety Evaluation Benchmark for Multimodal Large Language Models},
  author={Zheng, Baolin and Chen, Guanlin and Zhong, Hongqiong and Teng, Qingyang and Tan, Yingshui and Liu, Zhendong and Wang, Weixun and Liu, Jiaheng and Yang, Jian and Jing, Huiyun and others},
  journal={arXiv preprint arXiv:2505.23793},
  year={2025}
}

@article{gu2024mllmguard,
  title={{MLLMGuard}: A Multi-Dimensional Safety Evaluation Suite for Multimodal Large Language Models},
  author={Gu, Tianle and Zhou, Zeyang and Huang, Kexin and Dandan, Liang and Wang, Yixu and Zhao, Haiquan and Yao, Yuanqi and Yang, Yujiu and Teng, Yan and Qiao, Yu and others},
  journal={Advances in Neural Information Processing Systems},
  volume={37},
  pages={7256--7295},
  year={2024}
}

@article{ying2024safebench,
  title={{SafeBench}: A Safety Evaluation Framework for Multimodal Large Language Models},
  author={Ying, Zonghao and Liu, Aishan and Liang, Siyuan and Huang, Lei and Guo, Jinyang and Zhou, Wenbo and Liu, Xianglong and Tao, Dacheng},
  journal={arXiv preprint arXiv:2410.18927},
  year={2024}
}

@article{hu2024vlsbench,
  title={{VLSBench}: Unveiling Visual Leakage in Multimodal Safety},
  author={Hu, Xuhao and Liu, Dongrui and Li, Hao and Huang, Xuanjing and Shao, Jing},
  journal={arXiv preprint arXiv:2411.19939},
  year={2024}
}

@article{li2024mossbench,
  title={{MOSSBench}: Is Your Multimodal Language Model Oversensitive to Safe Queries?},
  author={Li, Xirui and Zhou, Hengguang and Wang, Ruochen and Zhou, Tianyi and Cheng, Minhao and Hsieh, Cho-Jui},
  journal={arXiv preprint arXiv:2406.17806},
  year={2024}
}

@article{pi2024mllm,
  title={{MLLM-Protector}: Ensuring {MLLM}'s Safety without Hurting Performance},
  author={Pi, Renjie and Han, Tianyang and Zhang, Jianshu and Xie, Yueqi and Pan, Rui and Lian, Qing and Dong, Hanze and Zhang, Jipeng and Zhang, Tong},
  journal={arXiv preprint arXiv:2401.02906},
  year={2024}
}

@inproceedings{cai2025safe,
  title={Safe Semantics, Unsafe Interpretations: Tackling Implicit Reasoning Safety in Large Vision-Language Models},
  author={Cai, Wei and Zhao, Jian and Jiang, Yuchu and Zhang, Tianle and Li, Xuelong},
  booktitle={Proceedings of the 33rd ACM International Conference on Multimedia},
  pages={13489--13491},
  year={2025}
}

@article{cai2025srpo,
  title={When Safe Unimodal Inputs Collide: Optimizing Reasoning Chains for Cross-Modal Safety in Multimodal Large Language Models},
  author={Cai, Wei and Liu, Shujuan and Zhao, Jian and Shi, Ziyan and Zhao, Yusheng and Yuan, Yuchen and Zhang, Tianle and Zhang, Chi and Li, Xuelong},
  journal={arXiv preprint arXiv:2509.12060},
  year={2025}
}

@article{zhu2025safemt,
  title={{SafeMT}: Multi-Turn Safety for Multimodal Language Models},
  author={Zhu, Han and Dai, Juntao and Ji, Jiaming and Li, Haoran and Cai, Chengkun and Wen, Pengcheng and Chan, Chi-Min and Chen, Boyuan and Yang, Yaodong and Han, Sirui and Guo, Yike},
  journal={arXiv preprint arXiv:2510.12133},
  year={2025}
}

@article{huang2025llavashield,
  title={{LLaVAShield}: Safeguarding Multimodal Multi-Turn Dialogues in Vision-Language Models},
  author={Huang, Guolei and Peng, Qinzhi and Xu, Gan and Huang, Yao and Lu, Yuxuan and Shen, Yongjun},
  journal={arXiv preprint arXiv:2509.25896},
  year={2025}
}

@article{liu2026mtmcs,
  title={{MTMCS-Bench}: Evaluating Contextual Safety of Multimodal Large Language Models in Multi-Turn Dialogues},
  author={Liu, Zheyuan and Kim, Dongwhi and Wan, Yixin and Yuan, Xiangchi and Tan, Zhaoxuan and Mo, Fengran and Jiang, Meng},
  journal={arXiv preprint arXiv:2601.06757},
  year={2026}
}

@article{zhu2026mintsafe,
  title={Towards Multi-Modal Multi-Turn Safety: From Agentic Interaction to Strategic Alignment},
  author={Zhu, Han and Chen, Jiale and Cai, Chengkun and Sun, Shengjie and Li, Haoran and Zhou, Yujin and Chan, Chi-Min and Wen, Pengcheng and Li, Lei and Guo, Yike and Han, Sirui},
  journal={arXiv preprint arXiv:2601.04736},
  year={2026}
}

@article{fan2025missing,
  title={Missing Premise Exacerbates Overthinking: Are Reasoning Models Losing Critical Thinking Skill?},
  author={Fan, Chenrui and Li, Ming and Sun, Lichao and Zhou, Tianyi},
  journal={arXiv preprint arXiv:2504.06514},
  year={2025}
}

@inproceedings{chen2025not,
  title={Do {NOT} Think That Much for 2+3=? {On} the Overthinking of Long Reasoning Models},
  author={Chen, Xingyu and Xu, Jiahao and Liang, Tian and He, Zhiwei and Pang, Jianhui and Yu, Dian and Song, Linfeng and Liu, Qiuzhi and Zhou, Mengfei and Zhang, Zhuosheng and others},
  booktitle={International Conference on Machine Learning},
  year={2025}
}

@article{cuadron2025danger,
  title={The Danger of Overthinking: Examining the Reasoning-Action Dilemma in Agentic Tasks},
  author={Cuadron, Alejandro and Li, Dacheng and Ma, Wenjie and Wang, Xingyao and Wang, Yichuan and Zhuang, Siyuan and Liu, Shu and Schroeder, Luis Gaspar and Xia, Tian and Mao, Huanzhi and others},
  journal={arXiv preprint arXiv:2502.08235},
  year={2025}
}

@inproceedings{wu2025collabllm,
  title={{CollabLLM}: From Passive Responders to Active Collaborators},
  author={Wu, Shirley and Galley, Michel and Peng, Baolin and Cheng, Hao and Li, Gavin and Dou, Yao and Cai, Weixin and Zou, James and Leskovec, Jure and Gao, Jianfeng},
  booktitle={International Conference on Machine Learning},
  year={2025}
}

@article{andukuri2024star,
  title={{STaR-GATE}: Teaching Language Models to Ask Clarifying Questions},
  author={Andukuri, Chinmaya and Fr{\"a}nken, Jan-Philipp and Gerstenberg, Tobias and Goodman, Noah D},
  journal={arXiv preprint arXiv:2403.19154},
  year={2024}
}

@article{kuhn2022clam,
  title={{CLAM}: Selective Clarification for Ambiguous Questions with Generative Language Models},
  author={Kuhn, Lorenz and Gal, Yarin and Farquhar, Sebastian},
  journal={arXiv preprint arXiv:2212.07769},
  year={2022}
}

@article{kadavath2022language,
  title={Language Models (Mostly) Know What They Know},
  author={Kadavath, Saurav and Conerly, Tom and Askell, Amanda and Henighan, Tom and Drain, Dawn and Perez, Ethan and Schiefer, Nicholas and Hatfield-Dodds, Zac and DasSarma, Nova and Tran-Johnson, Eli and others},
  journal={arXiv preprint arXiv:2207.05221},
  year={2022}
}

@inproceedings{kim2023tree,
  title={Tree of Clarifications: Answering Ambiguous Questions with Retrieval-Augmented Large Language Models},
  author={Kim, Gangwoo and Kim, Sungdong and Jeon, Byeongguk and Park, Joonsuk and Kang, Jaewoo},
  booktitle={Proceedings of the 2023 Conference on Empirical Methods in Natural Language Processing},
  pages={996--1009},
  year={2023}
}

@inproceedings{faltings2023interactive,
  title={Interactive Text Generation},
  author={Faltings, Felix and Galley, Michel and Brantley, Kiant{\'e} and Peng, Baolin and Cai, Weixin and Zhang, Yizhe and Gao, Jianfeng and Dolan, William B},
  booktitle={Proceedings of the 2023 Conference on Empirical Methods in Natural Language Processing},
  pages={4450--4468},
  year={2023}
}

@article{erbacher2023circle,
  title={{CIRCLE}: Multi-Turn Query Clarifications with Reinforcement Learning},
  author={Erbacher, Pierre and Soulier, Laure},
  journal={arXiv preprint arXiv:2311.02737},
  year={2023}
}

@inproceedings{zhang2024clamber,
  title={{CLAMBER}: A Benchmark of Identifying and Clarifying Ambiguous Information Needs in Large Language Models},
  author={Zhang, Tong and Qin, Peixin and Deng, Yang and Huang, Chen and Lei, Wenqiang and Liu, Junhong and Jin, Dingnan and Liang, Hongru and Chua, Tat-Seng},
  booktitle={Proceedings of the 62nd Annual Meeting of the Association for Computational Linguistics (Volume 1: Long Papers)},
  pages={10746--10766},
  year={2024}
}

@article{chen2026pir,
  title={Reasoning While Asking: Transforming Reasoning Large Language Models from Passive Solvers to Proactive Inquirers},
  author={Chen, Xin and Jiang, Feng and Zhang, Yiqian and Chen, Hardy and Yan, Shuo and Xie, Wenya and Yang, Min and Huang, Shujian},
  journal={arXiv preprint arXiv:2601.22139},
  year={2026}
}

@inproceedings{malinin2021uncertainty,
  title={Uncertainty Estimation in Autoregressive Structured Prediction},
  author={Malinin, Andrey and Gales, Mark},
  booktitle={International Conference on Learning Representations},
  year={2021}
}

@article{ghadimi2013stochastic,
  title={Stochastic First- and Zeroth-Order Methods for Nonconvex Stochastic Programming},
  author={Ghadimi, Saeed and Lan, Guanghui},
  journal={SIAM Journal on Optimization},
  volume={23},
  number={4},
  pages={2341--2368},
  year={2013}
}

@inproceedings{cai2026visual,
  title={Visual Attention Reasoning via Hierarchical Search and Self-Verification},
  author={Cai, Wei and Zhao, Jian and Yuan, Yuchen and Zhang, Tianle and Zhu, Ming and Tang, Haichuan and Li, Xuelong},
  booktitle={Proceedings of the 64th Annual Meeting of the Association for Computational Linguistics (Volume 1: Long Papers)},
  pages={8986--8997},
  year={2026}
}

\end{document}